\documentclass[a4paper, 11pt]{article}
\pdfoutput=1
\usepackage{latexsym,amsfonts, amsmath, amssymb, amsthm}
\usepackage[utf8]{inputenc}
\usepackage[final=true,protrusion=true,expansion=true]{microtype}
\usepackage[noadjust]{cite}
\usepackage{enumitem}
\usepackage{color}
\usepackage{mathtools}

\usepackage{booktabs}
\usepackage{multirow}

\usepackage{subcaption}
\usepackage[font=scriptsize,labelfont=bf]{caption}
\usepackage[affil-it]{authblk}

\usepackage{geometry}
\usepackage{hyperref}
\hypersetup{
	final,
    colorlinks=true,
    linkcolor=blue,
    filecolor=magenta,
    urlcolor=cyan,
}

\theoremstyle{plain}
\newtheorem{theorem}{Theorem}
\newtheorem{proposition}[theorem]{Proposition}
\newtheorem{lemma}[theorem]{Lemma}
\newtheorem{corollary}[theorem]{Corollary}

\theoremstyle{plain}
\newtheorem{definition}[theorem]{Definition}

\newtheorem{remark}[theorem]{Remark}

\providecommand{\Dim}{\operatorname{dim}}            
\providecommand{\dim}{\Dim}
\providecommand{\Vol}{\operatorname{Vol}}   

\providecommand{\ker}{\operatorname{ker}}

\providecommand*{\dist}[2]{\operatorname{dist}({#1};{#2})}   
\providecommand*{\normaldist}[2]{\operatorname{m}({#1,#2})}   
\providecommand*{\normaldistnoargs}{\operatorname{m}}   

\renewcommand{\Im}{\operatorname{Im}}             
\providecommand{\argmin}{\operatorname*{argmin}}  
\providecommand{\Id}{\Op{Id}}                     

\providecommand{\CA}{{\cal A}}

\providecommand{\CC}{{\cal C}}

\providecommand{\CM}{{\cal M}}
\providecommand{\CN}{{\cal N}}
\providecommand{\CO}{{\cal O}}
\providecommand{\CP}{{\cal P}}

\providecommand{\CZ}{{\cal Z}}

\providecommand{\bbE}{\mathbb{E}}

\providecommand{\bbN}{\mathbb{N}}

\providecommand{\bbR}{\mathbb{R}}

\providecommand*{\N}[1]{\left\|{#1}\right\|} 
\providecommand*{\textN}[1]{\|{#1}\|} 

\newcommand*{\SN}[1]{\left|{#1}\right|}      
\newcommand*{\textSN}[1]{|{#1}|}      

\newcommand*{\Op}[1]{\mathsf{#1}} 

\usepackage{bbm} 

\usepackage{ifdraft}
\usepackage{arydshln}

\newcommand{\reach}{\tau_{\CM}}
\newcommand{\medial}{\textrm{Med}(\CM)}

\newcommand{\lreach}[1]{\tau_{\CM}{(#1)}}
\newcommand{\relu}[1]{\left({#1}\right)_+}
\newcommand{\boundParams}[1]{B{(	#1)}}
\newcommand{\boundParamsnoargs}{B}

\newcommand{\expmet}{p}

\makeatother

\title{
		\usefont{OT1}{bch}{b}{n}
		\huge A deep network construction that adapts to intrinsic dimensionality beyond the domain \\
}

\date{}
\author[1,2]{Alexander Cloninger\thanks{Email: \texttt{acloninger@ucsd.edu}}}
\author[1,3]{Timo Klock\thanks{Email: \texttt{timo@simula.no}}}
\affil[1]{University of California, San Diego, Department of Mathematics, San Diego, US}
\affil[2]{Hal\i c\i \v{g}lu Data Science Institute, University of California San Diego, US}
\affil[3]{Simula Research Laboratory, Machine Intelligence Department, Oslo, Norway}

\begin{document}

\maketitle
\begin{abstract}
We study the approximation of two-layer compositions $f(x) = g(\phi(x))$ via deep networks with ReLU activation,
where $\phi$ is a geometrically intuitive, dimensionality reducing feature map. We focus
on two intuitive and practically relevant choices for $\phi$: the projection onto a low-dimensional embedded
submanifold and a distance to a collection of low-dimensional sets.
We achieve near optimal approximation rates, which  depend only on the complexity
of the dimensionality reducing map $\phi$ rather than the ambient dimension.
Since $\phi$ encapsulates all nonlinear features that are material to the function $f$,
this suggests that  deep nets are faithful to an intrinsic dimension governed by $f$ rather than
the complexity of the domain of $f$.
In particular, the prevalent assumption of approximating functions on low-dimensional manifolds
can be significantly relaxed using functions of type $f(x) = g(\phi(x))$ with $\phi$ representing an orthogonal projection onto the same manifold.
\end{abstract}


\section{Introduction}
\label{sec:introduction}
In the past decade neural networks emerged as powerful tools to construct state-of-the-art solutions
for various different data analysis tasks. Much of this progress is of
empirical nature and can not be explained by current mathematical theory. This led to a re-emerging interest
for developing a theoretical understanding of deep networks in recent years. In this work we contribute
to the effort by studying the approximative capacity of deep networks with respect to practically motivated composite
function classes in the high-dimensional regime.

Approximation properties of shallow and deep networks have been studied for over three decades and
gained much traction  during the rise of neural networks around the 80s and 90s
\cite{mhaskar1996neural, mhaskar1993approximation, leshno1993multilayer, cybenko1989approximation,hornik1989multilayer}.
It is well-known that shallow networks (with non-polynomial activation) are universal approximators,
which means they can approximate any continuous function
on a compact subset of $\bbR^D$ arbitrarily well \cite{leshno1993multilayer, cybenko1989approximation,hornik1989multilayer}.
Furthermore, it has been established that the number of required nonzero network parameters
for uniformly approximating a $\CC^{\alpha}$-function  to accuracy $\varepsilon$ on a compact subset of $\bbR^D$ is in $\CO(\varepsilon^{-D/\alpha})$ \cite{mhaskar1996neural,pinkus1999approximation}.
Similar results hold for deep networks with the
additional benefit that the approximation can be localized, contrary to approximation via shallow networks \cite{mhaskar1993approximation,chui1994neural,chui1996limitations}.

In modern networks differentiable sigmoidal activation functions are often replaced
by the recitified linear unit activation (ReLU), because such networks do not suffer
the vanishing gradient problem and can thus be more easily trained via backpropagation
\cite{goodfellow2016deep}. Approximation properties of ReLU networks received much attention
in recent years \cite{shaham2018provable,yarotsky2017error,petersen2018optimal,telgarsky2017neural,
yarotsky2018optimal,boelcskei2019optimal,grohs2019deep,shen2019nonlinear}.
The bottom line is that ReLU networks are at least as expressive as networks with differentiable sigmoidal activation.
Moreover, a series of recent works \cite{zhou2020universality,zhou2020theory,fang2020theory} shows that this is also true
for deep convolutional ReLU networks, which are significantly less flexible compared to fully-connected networks.
To comply with modern neural network practice, we concentrate on the ReLU activation in this work, though we emphasize that
we have no reason to believe our results are special to this choice.

Approximating functions either through  differentiable sigmoidal networks or ReLU networks suffers from the curse
of dimensionality, because the number of required parameters for approximating $f \in \CC^{\alpha}$ on a compact subset of $\bbR^D$
is exponential in $D$.
Since  high-dimensional problems are ubiquitous in applied areas,
it is of great interest to identify narrower but sufficiently rich
function classes that allow for faster approximation rates with at most polynomial dependency on $D$.

Three decades ago, the author of \cite{barron1993universal} showed that functions $f$,
whose Fourier transform $\hat f$ satisfies
\begin{align*}
C_f =  \int_{\bbR^D} \SN{\omega \hat f(\omega)}d\omega < \infty,
\end{align*}
can be approximated by a shallow network to accuracy $\varepsilon$ using just
$\CO(\varepsilon^{-2})$ neurons.
Functions satisfying such conditions are said to be of Barron-type and they are
under continuous investigation ever since \cite{barron1994approximation,klusowski2016uniform,montanelli2019deep}.
Unfortunately, the constant involved in $\CO(\varepsilon^{-2})$ depends on $C_f$, which in turn
increases exponentially with the dimension $D$ under standard
regularity assumptions alone. Several works \cite{mhaskar2004tractability,kurkova2001bounds,kurkova2002comparison} have subsequently investigated conditions on $f$ that imply
the growth of $C_f$ is at most polynomial in $D$.

In  \cite{poggio2015theory,poggio2017and,mhaskar2016learning,mhaskar2016deep,mhaskar2017and,schmidt2017nonparametric}
the benefit of depth of networks has been analyzed by studying approximation properties of deep nets
for compositional functions of the type $f(x) = g_L \circ \ldots \circ g_1(x)$.
Intuitively, if all intermediate functions $g_{\ell} : \bbR^{\ell-1}\rightarrow \bbR^{\ell}$
are easier to approximate than the final target $f$, deep networks can approximate $f$ more efficiently by mimicking
the compositional structure of the function. This situation arises, for instance, if
each component $g_{\ell,p}:\bbR^{\ell-1}\rightarrow \bbR$, $p =1,\ldots,\ell$, depends on at most
$k$ of the $\ell-1$ coordinates of the previous output, i.e., can be written as
$\tilde g_{\ell,p}(I_{\ell,p}(x)) = g_{\ell,p}(x)$ for a map $I_{\ell,p} : \bbR^{\ell-1}\rightarrow \bbR^k$
that selects $k$ coordinates, independently of $x$. In this case, assuming all components
$g_{\ell,p},p=1,\ldots,\ell,\ \ell = 1,\ldots,L$ are $\alpha$-H\"older, the function $f$ can be
approximated uniformly up to error $\varepsilon$ using $\CO(\varepsilon^{-k/\alpha})$
nonzero parameters (here, $L$ is treated as a constant).
The missing dependence on $D$ in the exponent show that compositions pave a way for defining classes of functions
that are narrow enough to avoid the curse of dimensionality \cite{mhaskar2016deep,poggio2017and,mhaskar2019function}. This led to the notion of `blessing of compositionality'
as a cure to the curse of dimensionality.

Another line of research, which is motivated by the popularity of nonlinear dimension reduction methods,
studies approximation of $f : \CM \subseteq [0,1]^D \rightarrow \bbR$ on low-dimensional
domains $\CM$, such as a $d$-dimensional embedded submanifold. The authors of \cite{shaham2018provable} established
that uniform approximations to accuracy $\varepsilon$ require just $\CO(\varepsilon^{-d/\alpha})$ parameters, replacing
the ambient dimension $D$ with the intrinsic manifold dimension $d$. Similar results have been shown in
\cite{chui2018deep, chen2019efficient, schmidt2019deep} and extended to more general notions of dimensionality
or other types of neural networks \cite{nakada2019adaptive,mhaskar2020dimension,mhaskar2020direct}, including radial basis function networks and
abstract generalizations thereof. Therefore, certain approximation systems, including
deep networks, adapt to the intrinsic dimension of the domain of the target.

Approximation on low-dimensional domains is appealing because it is geometrically
intuitive and can, to some extent, be checked in practice by analyzing local
covariance matrices of a given data set. However, defining the complexity of an approximation task
via the domain of the target has some significant drawbacks, which we highlight in the next section.

\subsection{Drawbacks of measuring complexity by the target domain}
\label{subsec:motivation}

\paragraph{Noisy manifold hypothesis}
Many theoretical results that alleviate the curse of dimensionality
are based either explicitly or implicitly on the exact manifold hypothesis, which states
that data is supported on a low-dimensional manifold.
In view of usually noisy real-world data, the exact manifold hypothesis
seems overly stringent and in fact has been criticized for being rarely observable
in practice \cite{hein2007manifold,hein2007manifoldb}. A more realistic alternative
is to model real-world data as a sum of clean data, which is supported on a low-dimensional manifold $\CM$
(think of the `face manifold' consisting of images of faces \cite{he2005face}), plus noise, which generically pushes data points off the clean data manifold.
If the noise is unstructured, we can simplistically assume that it concentrates in
the local normal space of $\CM$, and we may associate to $x \in \bbR^D$
the orthogonal projection $\pi_{\CM}(x) = \argmin_{z \in \CM}\N{x-z}_2$ as the clean data sample. We now aim for approximating functions
$f(x) = g(\pi_{\CM}(x))$, where $g : \CM \rightarrow \bbR$ describes a function of
interest defined on clean data. See Figure \ref{fig:motivation_1} for an illustration of the setting.

Following results in \cite{shaham2018provable, chen2019efficient, schmidt2019deep} about approximation over low-dimensional domains, we are tempted to think
there is a significant difference between approximating a function $g : \CM \rightarrow \bbR$ on $\CM$ or
a function $f(x) = g(\pi_{\CM}(x))$
on a full-dimensional tubular domain around $\CM$. We will prove that,
in fact, both functions are approximable with similarly sized networks and by using the same
amount of information about the target $f$.

We add that the stringency of the exact manifold hypothesis is often recognized and discussed in the literature.
For instance, the authors of \cite{chui2018deep} explain that their approximation results are robust to an inexact manifold hypothesis, because
noise that spreads only in  $s\ll D$ directions in the local normal space increases the dimensionality
of the data manifold to just $d+s\ll D$. Furthermore, \cite{mhaskar2020direct} proposes a
Hermite polynomial based
approximation scheme for functions on manifolds, which is
robust to a degree of off-manifold noise. The theory in \cite{cheng2019classification} includes off-manifold noise
under the assumption that the noise vanishes exponentially fast with increased distance from the manifold.

\paragraph{Adaptivity to function complexity}
The same argument as in the previous paragraph can be made when approximating a function that just depends on a lower dimensional set
of linear or nonlinear transformations of the input, as is common
in the sufficient dimension reduction literature \cite{li2018sufficient}. To give a simple example,
we may consider the swiss role manifold $\CM$ as in Figures \ref{fig:motivation_21}-\ref{fig:motivation_22}, where the colors indicate
values of two different Lipschitz-continous functions. Based on previously mentioned approximation results \cite{shaham2018provable, chen2019efficient, schmidt2019deep},
both functions can be approximated using deep networks with $\CO(\varepsilon^{-1/\dim(\CM)}) = \CO(\varepsilon^{-1/2})$ parameters.
However, the complexity of functions in \ref{fig:motivation_21} and \ref{fig:motivation_22} differs,
because we can express $f$ in \ref{fig:motivation_21} as $f(x) = g(\pi_{\gamma}(x))$,
where $\gamma$ is a one-dimensional manifold.
In other words, there exists a submanifold $\gamma \subset \CM$ with $\dim(\gamma) = 1$ that contains
all material information for recovering the target function $f$.

\begin{figure}
\begin{subfigure}{.49\textwidth}
  \centering
  \includegraphics[width=0.5\linewidth]{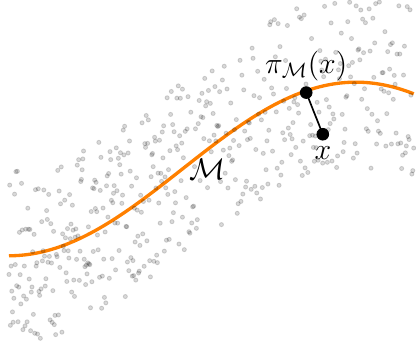}
  \caption{}
  \label{fig:motivation_1}
\end{subfigure}
\begin{subfigure}{.49\textwidth}
  \centering
  \includegraphics[trim=2.5cm 1.5cm 2.5cm 2cm, clip, width=0.7\linewidth]{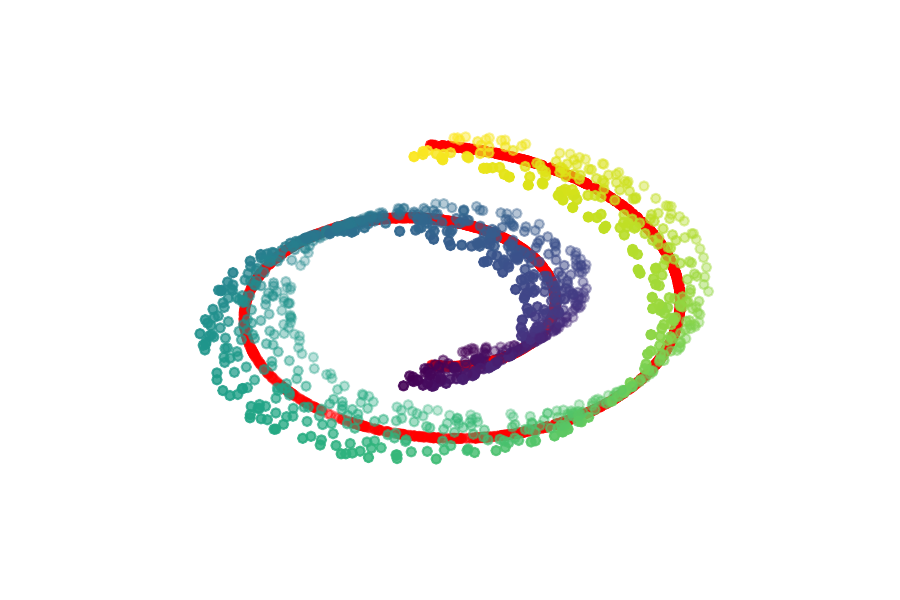}
  \caption{}
  \label{fig:motivation_21}
\end{subfigure}
\begin{subfigure}{.49\textwidth}
  \centering
  \includegraphics[trim=2.5cm 1.5cm 2.5cm 2cm, clip, width=0.7\linewidth]{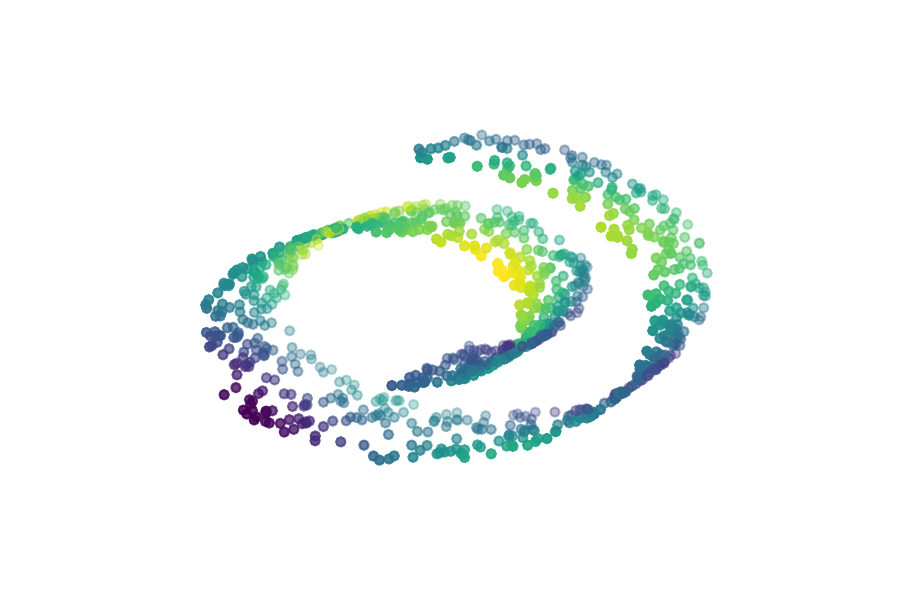}
  \caption{}
  \label{fig:motivation_22}
\end{subfigure}
\begin{subfigure}{.49\textwidth}
  \centering
  \includegraphics[width=.5\linewidth]{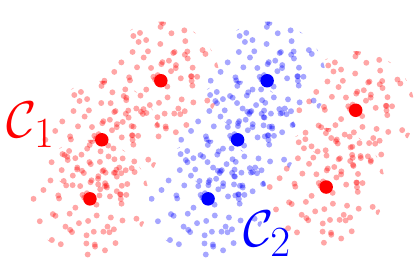}
  \caption{}
  \label{fig:motivation_3}
\end{subfigure}
\caption{
Examples highlighting drawbacks of defining the approximation complexity via the target domain.
In \ref{fig:motivation_1} the target function depends just on the projection of the input onto a low-dimensional manifold,
yet the data is spread in a full-dimensional subset of $\bbR^D$. \ref{fig:motivation_21} - \ref{fig:motivation_22} show
two functions whose domain is the swiss role, but which are of different complexity because the function in
\ref{fig:motivation_21} just depends on a single nonlinear transformation of the data (the red curve).
\ref{fig:motivation_3} shows a classification problem where labels are assigned based on the proximity to a
few class attractors (bold dots). In all three cases the dimensionality of the approximation
domain is not a suitable measure for the difficulty of the approximation problem.
}
\label{fig:motivation}
\end{figure}

\paragraph{Classification problems with class attractors}
Another example, where the domain of the target is not a suitable
measure of complexity, are classification problems with class attractors,
see Figure \ref{fig:motivation_3}. Here, we assume that the class label
depends only on the proximity
of the input to a low-dimensional attractor set, such as for instance a finite set of points.
Hence, if we were aware of the attractor set,
the target function is completely determined by evaluating the distance to the set,
indicating that the complexity of the target function is dictated by the complexity
of the distance metric and the set of attractors, rather than the domain of the target.
Classification problems, where a finite number of attractors exist, are the main
object of study in few-shot learning, see for instance \cite{sung2018learning}. In these problems
the goal is to predict class labels after querying a tiny amount of samples, which are ideally points that
serve as class attractors with respect to a, possibly prescribed, metric.

\subsection{Contribution}
\label{subsec:setting}
Our main goal is to extend approximation guarantees of deep nets
from functions defined on low-dimensional domains to functions that encode
low-dimensionality in the joint input-output relation $x \mapsto f(x)$.
We study two classes of functions, which resemble two layer composite functions $f(x) = g(\phi(x))$,
where $\phi(x)$ takes the role of a geometrically intuitive, dimensionality reducing feature map.
By resorting to such a function-driven notion of low-complexity,
we alleviate the drawbacks raised in the previous section.

\paragraph{Functions of projections to low-dimensional sets}
We first consider functions that model $\phi$ as an orthogonal projection
onto a $d$-dimensional Riemannian submanifold $\CM \subseteq [0,1]^D$.
In this case we can write the target $f : \CA \subseteq [0,1]^D \rightarrow \bbR$ as
\begin{align}
\label{eq:intro_model_projection}
f(x) = g(\pi_{\CM}(x))\quad \textrm{where}\quad \pi_{\CM}(x) \in \argmin_{z \in \CM}\N{x-z}_2,
\end{align}
and the approximation domain $\CA$ is assumed to be contained in a tubular region around $\CM$.
The width of this region is constrained
to guarantee that $\pi_{\CM}(x)$ is Lipschitz-continuous, as described in detail in Section \ref{sec:projections_onto_manifolds}.
We refer to the associated function class as Class 1 below.

Assumption \eqref{eq:intro_model_projection} naturally includes the popular
case $\CA = \CM$ and $\pi_{\CM} = \Id$, which has been studied in
\cite{shaham2018provable,chui2018deep, chen2019efficient, schmidt2019deep,nakada2019adaptive,mhaskar2020dimension,mhaskar2020direct}.
In the present case the approximation domain $\CA$ does however not need to be
low-dimensional. Rather, Equation \eqref{eq:intro_model_projection} imposes
that $f$ is locally constant in $D-d$ directions, corresponding
to the local normal space of $\CM$. If we were able to extract a subset of the approximation space $A\subseteq \CA$,
whose projection $\pi_{\CM}(A)$ is supported on a small patch of the manifold $\CM$ so that curvature effects of $\CM$ are negligible, we can
view $f|_A$ as a constant function with optimal regularity in $D-d$ directions
corresponding to the local normal space, and regularity dictated by $g|_{\pi_{\CM}(A)}$
in the remaining $d$ directions. Following this intuition, our viewpoint is aligned with recent
work on approximation of functions in anisotropic Besov spaces \cite{suzuki2018adaptivity,suzuki2019deep}.
%
%
%
%

\textbf{Contribution} We achieve the same approximation guarantee that is achieved in \cite{shaham2018provable,schmidt2019deep,nakada2019adaptive}
for the case $\CA = \CM$. Namely, if $\CM$ is a $d$-dimensional manifold satisfying
some common regularity assumptions and $g$ is $\alpha$-H\"older with respect to the geodesic metric
on $\CM$, functions of Class 1 can be approximated uniformly to accuracy
$\varepsilon$ using a deep ReLU network based on $\CO(\varepsilon^{-d/\alpha})$ point queries of $f$ and with $\CO(\log(D)D\log^2(\varepsilon^{-1})\varepsilon^{-d/\alpha})$
nonzero parameters arranged in $\CO(\log(D)\log^2(\varepsilon^{-1}))$ layers. The result is
optimal in terms of the number of required function queries according to nonlinear width theory \cite{devore1989optimal},
and optimal (apart from logarithmic factors) in terms of the required network dimensions \cite[Theorem 1]{yarotsky2018optimal}.
We believe the result sheds a new light on the relevance of the manifold hypothesis, because we identify
local invariances encoded in $x\mapsto f(x)$ as the key factor to simplify the approximation problem,
as opposed to the complexity of the underlying data manifold.

\paragraph{Functions of distances to low-dimensional sets}
Second, we study functions that depend only on distances to a collection of finite or low-dimensional sets
$\CC_1,\ldots,\CC_{M}$. Mathematically, we assume $f : [0,1]^D\rightarrow \bbR$ can be written as
\begin{align}
\label{eq:intro_model_distance}
f(x) = \sum_{\ell=1}^{M}g_{\ell}\left(\min_{z \in \CC_{\ell}}\normaldist{x}{z}^p\right),
\end{align}
where $\normaldist{\cdot}{\cdot}$ is a metric and $p \in \bbN$ can be an arbitrary scalar, which makes
$\normaldist{\cdot}{\cdot}^p$ efficiently approximable by deep neural networks (think
of $\normaldist{\cdot}{\cdot}^p = \N{\cdot - \cdot}_{p}^p$, which is a polynomial of degree $p$
in the coordinates and thus efficiently approximable, see Lemma \ref{lem:relu_lp_norm}).
For functions satisfying \eqref{eq:intro_model_distance}, low-dimensionality will be encoded
by assuming that packings of $\CC_1,\ldots,\CC_M$ at scale $\varepsilon$ with respect to $\normaldist{\cdot}{\cdot}$ have cardinality $\CO(\varepsilon^{-d})$.
This morally says $\CC_1,\ldots,\CC_M$ are $d$-dimensional submanifolds, though we do not require
any regularity about $\CC_\ell$ and we also cover the case $d = 0$. The associated function class is referred to
as Class 2 below.

\textbf{Contribution}
For $\alpha$-H\"older smooth $g_1,\ldots,g_M$,
we show that functions of type \eqref{eq:intro_model_distance} can be
uniformly approximated to accuracy $\varepsilon$  with ReLU nets based on $\CO(\varepsilon^{-\alpha})$ queries from each $g_1,\ldots,g_M$
and with $\CO\left(\log(\varepsilon^{-1})\varepsilon^{-\min\{1,d\}/\alpha )} +
\varepsilon^{-d/\alpha} P_{\normaldistnoargs}(\varepsilon^{1/\alpha})\right)$ nonzero network parameters.
Here, $P_{\normaldistnoargs}(\varepsilon)$ describes
the number of nonzero parameters required to uniformly approximate $\normaldist{\cdot}{\cdot}^p$ to accuracy $\varepsilon$.
If the metric can be efficiently approximated by a deep net, e.g., by bounding $P_{\normaldistnoargs}(\varepsilon) \in \CO(D\log(D)\log(\varepsilon^{-1}))$
such as in the case $\normaldist{\cdot}{\cdot}^p = \N{\cdot - \cdot}_{p}^p$,
we require in total $\CO(D\log(D)\varepsilon^{-\min\{1,d\}/\alpha})$ parameters in the network.
For $d \leq 1$, which corresponds to the situation in Figure \ref{fig:motivation_3},
the associated requirement $\CO(D\log(D)\log(\varepsilon^{-1})\varepsilon^{-1/\alpha})$ is
comparable to approximating a univariate function with a shallow or deep network \cite{mhaskar1996neural,yarotsky2017error,yarotsky2018optimal}.
Similarly, the number of required function queries $\CO(\varepsilon^{-\alpha})$ per $g_i,\ i=1,\ldots,M$, matches
the minimal number of queries needed to approximate an arbitrary $\alpha$-H\"older univariate functions
according to nonlinear width theory \cite{devore1989optimal}.

\subsection{Organization of the paper}
Section \ref{sec:projections_onto_manifolds} rigorously introduces functions of type \eqref{eq:intro_model_projection}
and presents the corresponding approximation guarantee. Section \ref{sec:approximation_distance_operator}
does the same for functions of type \eqref{eq:intro_model_distance}. Section \ref{sec:summary_statitics}
presents implications of our results to nonparametric estimation problems. Section \ref{sec:preliminaries}
introduces preparatory material about ReLU calculus and Sections
\ref{subsec:proofs_approximation_projection_new} and \ref{subsec:proof_approximation_distance} present the proofs of our main
results.
We conclude in Section \ref{sec:conclusions}. Section \ref{sec:appendix} in the Appendix
contains some additional statements and proofs about differential geometry and ReLU approximation theory.

\subsection{Notation}
\label{notation}
For $N \in \bbN$ we let $[N] := \{1,\ldots,N\}$.
$\textrm{cl}(B)$ denotes the closure of a set $B$ and $\Im(M)$ denotes the image
of an operator $M$.
$\SN{A}$ denotes the absolute value if $A \in \bbR$, the length if $A$ is an interval,
and the cardinality if $A$ is a finite set.
We denote $a \vee b = \max\{a,b\}$ and $a \wedge b = \min\{a,b\}$.
The ReLU activation function is denoted $\relu{t} = \max\{0,t\}$.

$\N{\cdot}_p$
denotes the standard Euclidean $p$-norm for vectors and $\N{\cdot}_2$ denotes the
spectral norm for matrices. We denote $\dist{z}{A} := \inf_{p \in A}\N{z - p}_2$ for $z \in \bbR^D$ and $A \subset \bbR^D$.
$B_{r}(x)$ denotes the standard $\N{\cdot}_2$-ball of radius $r$ around $x$, while $B_{\CM,r}(v)$ denotes
the geodesic ball on a manifold $\CM$ of radius $r$ around $v$. $\N{A}_0$ counts the number of nonzero entries
of a matrix $A$. $L_p(A)$ contains function with finite $p$-th order Lebesgue norm. 

We use $A \lesssim B$, respectively, $A\gtrsim B$, if there exists a
uniform constant $C$ such that $A\leq C B$, respectively $A\geq C B$. Furthermore, we
write $A\asymp B$ if $A\lesssim B$ and $A \gtrsim B$.

Finally, we define the ReLU activation function $\relu{t} = 0 \vee t = \max\{0,t\}$ and introduce
the following definition of a deep ReLU network.
\begin{definition}[{\cite[Definition 2.1]{grohs2019deep}}]
\label{def:relu_nets}
Let $L \geq 2$ and $N_0,\ldots,N_L \in \bbN_{>0}$.
A map $\Phi : \bbR^{N_0} \rightarrow \bbR^{N_{L}}$ is called a ReLU network if there
exist matrices $A_{\ell} \in \bbR^{N_{\ell}\times N_{\ell - 1}}$
and vectors $b_{\ell} \in \bbR^{N_{\ell}}$ for $\ell \in [L]$ so that
$\Phi(x) = W_{L}y_{L-1} + b_{L}$, where $y_{\ell}$ is recursively defined by $y_0 := x$ and
\begin{align*}
y_{\ell} := \relu{A_{\ell}y_{\ell-1} + b_{\ell}}\quad \textrm{for} \quad \ell \in [L-1].
\end{align*}
Furthermore, we define $L(\Phi) := L$ as the number of layers, $W(\Phi) := \max_{\ell = 0,\ldots,L}N_{\ell}$
as the maximum width, $P(\Phi) := \sum_{\ell=1}^{L} \N{A_{\ell}}_0 + \N{b_{\ell}}_0$ as the number of nonzero parameters,
and
$$
\boundParams{\Phi} := \max\{\SN{(b_{\ell})_i}, \SN{(A_{\ell})_{ij}} : i \in N_{\ell}, j \in N_{\ell - 1}, \ell \in [L]\}
$$
as a bound for the absolute value over all parameters.
\end{definition}

\section{Main result: projection-based target functions}
\label{sec:projections_onto_manifolds}
In this section we rigorously introduce projection-based functions as foreshadowed in \eqref{eq:intro_model_projection}
and we present the corresponding approximation guarantee. Before doing so, we introduce some
well-known preparatory concepts from differential geometry. These are also summarized in Table \ref{tab:notation}.

\paragraph{Preparatory material from differential geometry}
\begin{table}[t]
\begin{center}
\scriptsize
\begin{tabular}{@{}ccc@{}}
      symbol & description  \\ \toprule
$\CM$ & a connected compact $d$-dimensional Riemannian submanifold of $\bbR^D$ \\
$d$ & dimension of the manifold $\CM$\\
$\pi_{\CM}$ & orthogonal projection $\pi_{\CM}(x) = \argmin_{z \in \CM}\N{x-z}_2$ \\
$\medial$ & medial axis of $\CM$, i.e. set with non-unique projections $\pi_{\CM}(x)$ \\
$A(v)$ & $D\times d$ matrix containing columnwise orthonormal basis for the tangent space $\CM$ at $v$ \\
$\lreach{v}$ & local reach at $v \in \CM$, i.e. distance to travel in $\Im(A(v))^\perp$ to reach $\medial$ \\
$\reach$ & infimum over all local reaches, throughout assumed positive\\
$\CM(q)$ & tube of radius $q \in [0,1)$ times local reach around $\CM$, see \eqref{eq:blow_up_of_manifold}\\
$d_{\CM}(v,v')$ & geodesic metric on $\CM$\\
$d_{T}(v,v')$ & geodesic metric on $\CM$ extended to $T\supseteq \CM$ by $d_{T}(x,x'):=d_{\CM}(\pi_{\CM}(x),\pi_{\CM}(x'))$\\
$B_{\CM,r}(v)$ & geodesic ball of radius $r$ around $v\in \CM$\\
$\Vol(\CM)$ & volume of the manifold $\CM$\\
$\CP(\delta,\CC,\Delta)$ & $\delta$-packing number of a set $\CC$ with respect to metric $\Delta$\\
\bottomrule
\end{tabular}
\end{center}
\caption{Notations used for different geometrical concepts throughout the paper}
\label{tab:notation}
\end{table}
Let $\CM\subseteq \bbR^D$ be a nonempty, connected, compact, $d$-dimensional Riemannian submanifold.
A manifold $\CM$ has an associated medial axis
\begin{align}
\label{eq:medial_axis}
\medial := \left\{ x \in \bbR^D : \exists p \neq q \in \CM,\ \N{p-x}_2 = \N{q-x}_2 = \dist{x}{\CM}\right\},
\end{align}
which contains all points $x \in \bbR^D$ with set-valued orthogonal projection $\pi_{\CM}(x) = \argmin_{z \in \CM}\N{x-z}_2$.
The local reach (sometimes called local feature size \cite{boissonnat2014manifold})
is defined by
\begin{align}
\label{eq:local_reach}
\lreach{v} := \dist{v}{\medial}
\end{align}
and describes the minimum distance needed to travel from a point $v \in \CM$ to the closure
of the medial axis. The smallest local reach $\reach := \inf_{v \in \CM}\lreach{v}$
is called reach of $\CM$.

Another important concept, which we use in the following, is the geodesic metric.
Since compact Riemannian manifolds are geodesically complete by the  Hopf-Rinow theorem, there exists a length-minimizing
geodesic $\gamma : [t,t'] \rightarrow \CM$ between any two points $\gamma(t) = v$ and $\gamma(t')=  v'$,
where the length is defined by $\SN{\gamma} = \int_{t}^{t'}\N{\dot\gamma(s)}_2 ds$. The geodesic metric on $\CM$ is defined as
\begin{align}
\label{eq:geodesic_metric}
d_{\CM}(v,v') := \inf\{\SN{\gamma} : \gamma \in \CC^1([t,t']),\ \gamma: [t,t']\rightarrow \CM,\ \gamma(t) = v,\ \gamma(t') = v'\}.
\end{align}
We can extend $d_{\CM}$ to tubular regions $T\supseteq \CM$ around $\CM$ by
$d_{T}(x,x') := d_{\CM}(\pi_{\CM}(x),\pi_{\CM}(x'))$, provided the orthogonal projection $\pi_{\CM}$ is uniquely defined for $x,x' \in T$.

\paragraph{Main result}
We are now interested in approximating functions of the type  $f = g\circ \pi_{\CM}$. To state
the function class in rigorous terms, we define the set
\begin{align}
\label{eq:blow_up_of_manifold}
\CM(q) := \left\{ x \in \bbR^D : x = v + u,\ v \in \CM,\ u \in \ker(A(v)^\top),\ \N{u}_2 < q\lreach{v}\right\},
\end{align}
where the columns of $A(v) \in \bbR^{D\times d}$ represent an orthonormal basis of the tangent space of $\CM$ at $v$.
The set $\CM(q)$ represents a tubular region around the manifold $\CM$ with local tube radius
$q\lreach{v}$, where $\lreach{v}$ is the local reach as defined in \eqref{eq:local_reach}. Since $\lreach{v} \geq \reach$ for all $v \in \CM$,
$\CM(q)$ contains, for instance, the tube of constant radius $q\reach$ around $\CM$. However,
in regions where $\CM$ has small curvature, the tube radius may also be significantly larger due to its scaling with the local reach.

The class of projection-based functions is defined as follows.
\begin{enumerate}[leftmargin=1.5cm, label=Class \arabic*]
\item\label{enum:model_1}
The target $f : \CA \subseteq [0,1]^D \rightarrow \bbR$ can be written as $f(x) = g(\pi_{\CM}(x))$ for a connected, compact, nonempty, $d$-dimensional manifold $\CM$ with $\reach > 0$, $\CA \subseteq \CM(q)\subseteq [0,1]^D$ for some $q \in [0,1)$, and where $\pi_{\CM}(x) := \argmin_{z \in \CM}\N{x-z}_2$.
The function $g:\CM \rightarrow [0,1]$ is $\alpha$-H\"older with H\"older constant $L$, i.e., satisfies for $\alpha \in (0,1]$ and $L \geq 0$
\begin{equation}
\label{eq:holder_continuity}
\SN{g(v) - g(v')} \leq L d_{\CM}^{\alpha}(v,v') \quad \textrm{for all}\quad v,v' \in \CM.
\end{equation}
\end{enumerate}

The condition $\CA \subseteq \CM(q)$ for some $q<1$  is important because it is a necessary for $f$
to inherit smoothness properties from $g$. Namely, if $\CA$ intersects the medial
axis $\medial$, see the definition in \eqref{eq:medial_axis}, the projection $\pi_{\CM}$ is
not uniquely defined over $\CA$ and, as a consequence, $f$ may not be well-defined as well.
If  $\CA \cap \medial = \emptyset$ but $\dist{\CA}{\medial} = 0$,
$f$ might be well-defined and continuous on $\CA$, but we can not expect $f$ to be locally H\"older-continuous
at points arbitrarily close to the medial axis. As shown in the following Lemma, enforcing $\CA \subseteq \CM(q)$
for some $q < 1$ solves these issues and implies that $f$ inherits $\alpha$-H\"older regularity of
$g$ with a H\"older constant equal to the product of the H\"older constant of $g$ and $(1-q)^{-1}$.

\begin{lemma}
\label{lem:combined_unique_projection_lipschitz_property}
Consider a connected, compact, $d$-dimensional Riemannian submanifold of $\CM \subseteq \bbR^D$ with $\reach > 0$
and let $q \in [0,1)$.

\noindent
1) If $x \in \CM(q)$ has decomposition $x = v + u$ for $v \in \CM$ and $u \in \ker(A(v)^\top)$
with  $\N{u}_2 < q \lreach{v}$, then $\pi_{\CM}(x)$ is uniquely determined by $\pi_{\CM}(x)=v$.

\noindent
2) The projection $\pi_{\CM}$ satisfies $\N{\pi_{\CM}(x) - \pi_{\CM}(x')}_2 \leq (1-q)^{-1}\N{x-x'}_2$
for all $x,x' \in \CM(q)$.
\end{lemma}
\begin{proof}
The proof is deferred to Section \ref{subsec:proofs_sec_2} in the Appendix.
\end{proof}

We can now present our main approximation guarantee.
\begin{theorem}
\label{thm:approximating_function_projection}
Let $f$ be of \ref{enum:model_1} and define  $C_{\CM} := \textrm{Vol}(\CM)d^{d/2}$,
$C_{q} := C_d d^{d/2}(1-q)^{-2d}$, where $C_d$ is the volume of the Euclidean unit ball in $\bbR^{d}$.
For $\varepsilon \in (0, \reach/2)$ there exists a ReLU network $\Phi$, which uses
$n \lesssim C_{\CM}\varepsilon^{-d}$ point queries of $f$ and has its dimensions bounded according to
$\boundParams{\Phi} \lesssim \varepsilon^{-2}$, $W(\Phi)\lesssim DC_{\CM}\varepsilon^{-d}$,
and
\begin{equation}
\label{eq:achitecture_bounds_main}
\begin{aligned}
L(\Phi)&\lesssim C_{q}^4 \log^2\left(\frac{C_{q}}{\varepsilon^{\alpha}}\right) + \log\left(\frac{DC_{q}C_{\CM}}{\reach^2 \varepsilon^{3+d}}\right),\\
P(\Phi)&\lesssim C_{q}^4 C_{\CM}\varepsilon^{-d}\log^2\left(\frac{C_{q}}{\varepsilon^{\alpha}}\right) + D\varepsilon^{-d}\log\left(\frac{DC_{q}C_{\CM}}{\reach^2\varepsilon^{3+d}}\right),
\end{aligned}
\end{equation}
such that
\begin{align}
\label{eq:guarantee_function_projection}
\sup_{x \in \CA}\SN{f(x)-\Phi(x)} \lesssim \left(1 + \frac{L}{(1-q)^{2\alpha}} \right)\varepsilon^{\alpha}.
\end{align}
Alternatively, with access to $n \gtrsim (\reach/2)^d C_{\CM}$ point queries of $f$, we can construct a ReLU network
$\Phi$ (with dimensions as in \eqref{eq:achitecture_bounds_main} and $\varepsilon \asymp (C_{\CM}/n)^{1/d}$) that approximates $f$ up to
\begin{align*}
\sup_{x \in \CA}\SN{f(x)-\Phi(x)} \lesssim \left(1 + \frac{L}{(1-q)^{2\alpha}} \right)\left(\frac{C_\CM}{n}\right)^{\frac{\alpha}{d}}.
\end{align*}
The same construction can be achieved with a network $\tilde \Phi$ with $L(\tilde \Phi)\lesssim \log(\boundParams{\Phi})L(\Phi)$,
$W(\tilde \Phi)\lesssim (W(\Phi))^2$, $P(\tilde \Phi)\lesssim \log(\boundParams{\Phi})P(\Phi)$ and $\boundParams{\tilde \Phi}\leq 2$
according to \cite[Proposition A.1]{grohs2019deep}.
\end{theorem}
\begin{proof}
A proof sketch and full proof details are given in Section \ref{subsec:proofs_approximation_projection_new}.
\end{proof}

Theorem \ref{thm:approximating_function_projection} shows that functions of \ref{enum:model_1}
can be uniformly approximated to accuracy $\varepsilon$ with a budget of $\CO(\varepsilon^{-d/\alpha})$ queries of $f$ and a network with
$\CO(\log^2(\varepsilon^{-1})\varepsilon^{-d/\alpha})$ nonzero parameters arranged
in $\CO(\log(\varepsilon^{-1}))$ layers. Since the problem class contains $\alpha$-H\"older functions on $\bbR^d$, this result is optimal in terms of the number of needed function queries according to the theory of nonlinear width \cite{devore1989optimal}. Moreover, apart from logarithmic factors, it is optimal  in terms of the number of nonzero parameters in the network \cite[Theorem 1]{yarotsky2018optimal}. A bound for the number of nonzero parameters can be used to control covering numbers of the associated ReLU function spaces \cite[Lemma 5]{schmidt2017nonparametric}. Bounds for covering numbers can then be combined with statistical learning theory
to provide estimation guarantees for empirical risk minimization, see the details in Section \ref{sec:summary_statitics}. We also note that $W(\Phi)$ and $P(\Phi)$ have a mild
log-linear dependency on the ambient dimension $D$, which is possibly not avoidable apart
from cutting the log-factors.

The constant $C_{\CM}$ is intrinsic to $\CM$ and arises from bounding the cardinality of an $\varepsilon$-covering
of $\CM$ as in Lemma \ref{lem:auxiliary_results_diff_geom}. The constant $C_{q}$ and
the factor $(1-q)^{-1}$ in \eqref{eq:guarantee_function_projection} are extrinsic
as they depend on the approximation domain $\CA$ via $(1-q)^{-1}$. The factor $(1-q)^{-1}$ indicates
that approximating $f$ becomes increasingly challenging as $\dist{\CA}{\medial}$ shrinks, i.e., as
the approximation domain approaches the medial axis, where $\pi_{\CM}$ is set-valued and
$f$ loses regularity.

The number of needed queries of $f$ and the required dimension of the network in Theorem \ref{thm:approximating_function_projection} are,
apart from log-factors and constants, similar to the case $\CA = \CM$ and $\pi_{\CM} = \Id$ \cite{shaham2018provable,schmidt2019deep,nakada2019adaptive}.
Hence, previously studied function classes can be significantly extended
without compromising on the ability of deep networks to approximate them.

\begin{remark}\ \\
\label{rem:additional_comments}
1. Instead of defining $\CM$ implicitly by the target $f$ as in \ref{enum:model_1},
we can also start with a fixed manifold $\CM$,
an associated approximation domain $\CM(q)$ for $q \in [0,1)$, and ask how well all functions of the type
$f(x) = g(\pi_{\CM}(x))$ can be approximated over $\CM(q)$. Theorem \ref{thm:approximating_function_projection} applies to this case as well.
Furthermore, we note that all weights except for the last layer are used for the approximation of $\pi_{\CM}$ in our
construction. Therefore, if we approximate two functions $f(x) = g(\pi_{\CM}(x))$ and $\widetilde f(x) = \widetilde g(\pi_{\CM}(x))$
using the proposed construction, the associated networks differ only in the last layer.\\
2. As the proof in Section \ref{subsec:proofs_approximation_projection_new} will show, there is no significant advantage of the ReLU activation for the construction
of the approximating network. Therefore, we believe that similar constructions are realizable with other common activation functions.
We focus on the ReLU in this work simply because it is the most prominent choice in practice.\\
3. The results of Theorem \ref{thm:approximating_function_projection} are achieved with networks that have
duplicate weights, for the sake of an easier analysis. Removing duplicate weights only affects the constant factors in the bounds
of Theorem \ref{thm:approximating_function_projection}.
\end{remark}

As a corollary of Theorem \ref{thm:approximating_function_projection}, we can also derive an approximation guarantee for $\pi_{\CM}$.

\begin{corollary}
\label{cor:approximation_of_projection_operator}
Let $q \in [0,1)$ and let $\CM$ be a nonempty, connected, compact $d$-dimensional manifold with $\reach > 0$ and $\CM(q) \subseteq [0,1]^D$.
For $\varepsilon \in (0,\reach/2)$ there exists a ReLU network $\Phi$ with architecture constrained
as in Theorem \ref{thm:approximating_function_projection} and
\begin{align}
\sup_{x\in \CM(q)}\N{\pi_{\CM}(x) - \Phi(x)}_{\infty}\lesssim\varepsilon.
\end{align}
\end{corollary}
\begin{proof}
The proof is given at the end of Section \ref{subsec:proofs_approximation_projection_new}.
\end{proof}

\section{Main result: distance-based target functions}
\label{sec:approximation_distance_operator}
We now study distance-based target functions as foreshadowed by Equation \eqref{eq:intro_model_distance}.
The rigorous definition of the function class requires the
well-known concept of  packing numbers.
\begin{definition}[{\cite[Section 4.2]{vershynin2018high}}]
\label{def:geoodesic_covering_number}
Let $\CC$ be a set endowed with a metric $\Delta$ and let $\delta > 0$.
We say $\CZ \subset \CC$ is $\delta$-separated if for any $z\neq z' \in \CZ$
we have $\Delta(z,z') > \delta$. $\CZ$ is maximal separated if adding any other
point in $\CZ$ destroys the separability property. The cardinality of the largest maximal separated set
is called the packing number and denoted by $\CP(\delta,\CZ,\Delta)$.
\end{definition}

\begin{enumerate}[leftmargin=1.5cm, label=Class \arabic*]
\setcounter{enumi}{1}
\item\label{enum:model_2} Let $\CC_{1},\ldots,\CC_{M} \subseteq [0,1]^D$  be nonempty closed sets,
let $\normaldist{\cdot}{\cdot} : [0,1]^D \rightarrow [0,1]$ be a continuous (normalized) metric,
and assume there exists $\delta_0 > 0$ such that
$\CP(\delta, \CC_{\ell}, \normaldistnoargs) \lesssim \delta^{-d}$ for all $\delta < \delta_0$ and $\ell \in [M]$.
Furthermore, assume there exists $p > 0$ so that $\normaldistnoargs^\expmet$ is ReLU-approximable in the sense that,
for any fixed $z \in [0,1]^D$ and $\varepsilon > 0$, there exists a ReLU net $\Psi_{z,\varepsilon}$
with $\sup_{x \in [0,1]^D}\SN{\normaldist{x}{z}^\expmet - \Psi_{z,\varepsilon}(x)} \leq \varepsilon$ and
\begin{equation}
\begin{aligned}
\label{eq:assumption_distance_relu_approximable}
L(\Psi_{z,\varepsilon}) \leq &L_{\normaldistnoargs}(\varepsilon),\ \  W(\Psi_{z,\varepsilon}) \leq W_{\normaldistnoargs}(\varepsilon),
\ \ P(\Psi_{z,\varepsilon}) \leq P_{\normaldistnoargs}(\varepsilon),\ \ B(\Psi_{z,\varepsilon}) \leq B_{\normaldistnoargs}(\varepsilon).
\end{aligned}
\end{equation}
We consider functions of the form $f(x) = \sum_{\ell=1}^{M}g_{\ell}\left(\min_{z \in \CC_{\ell}}\normaldist{x}{z}^\expmet\right)$,
with $g_{\ell} : [0,1]\rightarrow [0,1]$ satisfying for some $\alpha \in (0,1]$
\begin{equation}
\label{eq:hoelder_condition_one_d}
\SN{g_{\ell}(t) - g_{\ell}(t')} \leq L\SN{t-t'}^{\alpha}\qquad \textrm{for all }\qquad t,t'\in [0,1].
\end{equation}
\end{enumerate}

The parameter $\expmet \geq 1$ in \ref{enum:model_2} can be useful for making functions
$\normaldist{x}{z}^\expmet$ more easily approximable compared to $\normaldist{x}{z}$
(think of $p$-th order Euclidean norms raised to the power $p$, which are degree $p$ polynomials and can be easily approximated as shown in Lemma \ref{lem:relu_lp_norm}).
Furthermore, \eqref{eq:assumption_distance_relu_approximable}
should be seen as a definition of $L_{\normaldistnoargs}(\varepsilon),W_{\normaldistnoargs}(\varepsilon),
P_{\normaldistnoargs}(\varepsilon),
B_{\normaldistnoargs}(\varepsilon)$ rather than as an assumption, because it poses almost no restriction on
the metric $\normaldistnoargs$ in view of universal approximation theorems. However, if the approximation of $\normaldistnoargs^\expmet$
is responsible for an overwhelming majority of the required nonzero parameters in the network construction or scales
exponentially in $D$, the corresponding metric $\normaldistnoargs$ does not induce an interesting function class in the sense
of reducing the original complexity of approximating $f$. We return to this point after stating the main result
by discussing some practically relevant metrics $\normaldistnoargs$.

\begin{theorem}
\label{thm:approximating_function_distance}
Let $f$ be a function of \ref{enum:model_2}. For any $\varepsilon \in (0,2p\delta_0)$ there exists a ReLU network $\Phi$, which uses
$n \lesssim \varepsilon^{-1}$ point queries from each $g_1,\ldots,g_M$ and has its dimensions bounded according to
\begin{align*}
L(\Phi)&\lesssim  d\log(p\varepsilon^{-1}) + L_{\normaldistnoargs}(\varepsilon),\\
W(\Phi)&\lesssim M\expmet^d\varepsilon^{-(1 \vee d)}W_{\normaldistnoargs}(\varepsilon)\\
P(\Phi)&\lesssim M p^dd\log(p\varepsilon^{-1})\varepsilon^{-(1\vee d)} + M p^d\varepsilon^{- d}P_{\normaldistnoargs}(\varepsilon)
\end{align*}
and $B(\Phi) \leq 1 \vee \boundParamsnoargs_{\normaldistnoargs}(\varepsilon)$,
such that
\begin{align}
\label{eq:guarantee_function_distance}
\sup_{x\in[0,1]^D}\SN{f(x)-\Phi(x)} \lesssim M L \varepsilon^{\alpha}.
\end{align}
Alternatively, with access to $n$ point queries from each $g_1,\ldots,g_M$, we can construct a ReLU network
$\Phi$ (with dimensions as above and $\varepsilon \asymp n^{-1}$) that approximates $f$ up to
\begin{align*}
\sup_{x \in \CA}\SN{f(x)-\Phi(x)} \lesssim \frac{M L}{n^{\alpha}}.
\end{align*}
\end{theorem}
\begin{proof}
The proof is deferred to Section \ref{subsec:proof_approximation_distance}.
\end{proof}

As long as $L_{\normaldistnoargs}(\varepsilon),\ W_{\normaldistnoargs}(\varepsilon)$, and $P_{\normaldistnoargs}(\varepsilon)$
grow at most polylogarithmically in $\varepsilon^{-1}$ and possibly polynomially in $D$, Theorem \ref{thm:approximating_function_distance}
shows that their contribution to the overall network complexity is negligible. Specifically,
Theorem \ref{thm:approximating_function_distance} then implies approximation of
$f$ to accuracy $\varepsilon$ using $\CO(\varepsilon^{-1})$ queries from each $g_1,\ldots,g_M$ and
$\CO(\textrm{polylog}(\varepsilon^{-1})\varepsilon^{-(1 \vee d)})$ nonzero parameters arranged in
$\CO(\textrm{polylog}(\varepsilon^{-1}))$ layers.
If $M = 1$, querying $g_1$ is similar to querying $f$ and the result is optimal according
to the theory of nonlinear width \cite{devore1989optimal}.
Moreover, if $d\leq 1$ and if we neglect logarithmic factors, the number of required nonzero parameters
is optimal among all networks whose depth grows at most logarithmically in $\varepsilon^{-1}$ \cite[Theorem 1]{yarotsky2018optimal}.
We remark that 2. and 3. of Remark \ref{rem:additional_comments} about the importance of the ReLU activation
and the use of weight duplication apply to Theorem \ref{thm:approximating_function_distance} as well.

Metrics induced by $L_p$-norms present a practical and versatile instance of
metrics that can be efficiently approximated by deep networks.
Specifically, we require $\CO(D\log(D/\varepsilon))$
nonzero parameters, arranged in $\CO(\log(D/\varepsilon))$ layers as shown in Lemma \ref{lem:relu_lp_norm},
so that the overall number of nonzero parameters of the approximating network equals
$\CO(D\log(D/\varepsilon)\varepsilon^{-(1 \vee d)})$ ($L_1$ and $L_{\infty}$ are actually
exactly realizable with smaller networks, see also Remark \ref{rem:L_infty_appprox}). We can also consider
variations of $L_p$-norms, for instance by first transforming inputs through
a sparsity inducing basis (e.g., a wavelet transformation operator) and then use an $L_1$-norm,
or by considering weighted sums of multiple $L_p$-norms, where each $L_p$-norm measures the discrepancy
of two points at different scales. To give a concrete example, we refer to the work \cite{shirdhonkar2008approximate,leeb2016holder}, who
approximate the earth movers distance for histograms using a weighted sum of $L_1$-norms of
wavelet coefficients of histogram differences.

We also note that \ref{enum:model_2} contains radial functions with $M=1$, $d = 0$, and $\normaldist{\cdot}{\cdot} = \N{\cdot-\cdot}_2^2$.
\cite{chui2019deep} proves an approximation rate for radial functions similar to ours using smooth activation functions
and \cite{mccane2017deep} shows  dimension-free but sub-optimal rates for ReLU networks.
Interestingly, \cite{chui2019deep} also proves that shallow networks can not achieve
dimension-free rates, because they can not leverage the compositional nature of $f$.

\section{Implications on nonparametric estimation problems}
\label{sec:summary_statitics}
In this section we briefly highlight some implications of our results
on nonparametric estimation problems. We will focus on regression problems with
$X$ being a random input vector in $\bbR^D$, $Y = f(X) + \zeta$, and $\bbE[\zeta|X] = 0$.
Furthermore, we assume $f$ is of \ref{enum:model_1} or \ref{enum:model_2} (where the metric $\normaldistnoargs$ is
assumed to be as efficiently approximable as $L_p$-norms by a deep ReLU net).

Several very recent works \cite{bauer2019deep,schmidt2017nonparametric,schmidt2019deep}
studied the performance of the empirical risk minimizer
\begin{align}
\label{eq:erm_def}
\hat \Phi \in \argmin_{\hat \Psi \in \CN_N}\sum_{i=1}^{N}\left(\hat \Psi(X_i)- Y_i\right)^2,
\end{align}
where the hypothesis space $\CN_N $ contains ReLU networks $\hat \Psi$ with complexity bounded by $L(\hat \Psi) \leq L_N$, $W(\hat \Psi) \leq W_N$, $P(\hat \Psi) \leq P_N$,
$B(\hat \Psi) \leq B_N$, and $L_N,W_N,P_N,B_N$ depend on the size of the training data $\{(X_i,Y_i): i \in [N]\}$.
The complexity of $\CN_N$ can be controlled in terms of $L_N, W_N, P_N$ and $B_N$ \cite[Lemma 5]{schmidt2017nonparametric}, and
a bias-variance tradeoff analysis allows for establishing estimation rates for \eqref{eq:erm_def},
whenever the approximation error $\inf_{\Psi \in \CN_N}\bbE\left(\Psi(X) - f(X)\right)^2$
can be bounded in terms of  $L_N,W_N,P_N$ and $B_N$.

Following this strategy, Theorems \ref{thm:approximating_function_projection} and \ref{thm:approximating_function_distance}
can be used to derive the estimation guarantees
\begin{align}
\label{eq:erm_rate_example}
\bbE\left(\hat \Phi(X) - f(X)\right)^2 \in \begin{cases}
\tilde \CO(N^{-\frac{2\alpha}{2\alpha +d}}), & \textrm{ if } $f$ \textrm{ is of   \ref{enum:model_1}},\\
\tilde \CO(N^{-\frac{2\alpha}{2\alpha +(1\vee d)}}), & \textrm{ if } $f$ \textrm{ is of  \ref{enum:model_2}},\\
\end{cases}
\quad \textrm{as } N\rightarrow \infty,
\end{align}
where $\tilde \CO$ absorbs log-factors in $N$. The corresponding relations between the architectural constraints and the size of the training data $N$
are given by
\begin{equation}
\label{eq:network_constraints}
\begin{aligned}
L_N \in  \tilde \CO(1),\ P_N \in \tilde \CO\left(N^{\frac{d}{2\alpha+d}}\right),\ W_N \in \tilde \CO\left(N^{\frac{d}{2\alpha+d}}\right),\ B_N \in \tilde \CO\left(N^{\frac{2}{2\alpha+d}}\right),\textrm{ for \ref{enum:model_1}},\\
\textrm{and}\ \  L_N \in  \tilde \CO(1),\ P_N \in \tilde \CO\left(N^{\frac{1\vee d}{2\alpha+1\vee d}}\right),\ W_N \in \tilde \CO\left(N^{\frac{1\vee d}{2\alpha+1\vee d}}\right),\ B_N \in \tilde \CO\left(1\right),\textrm{ for \ref{enum:model_2}}.
\end{aligned}
\end{equation}
The rates in \eqref{eq:erm_rate_example} are statistically minimax optimal for \ref{enum:model_1} (even if $X$ is supported
exactly on a $d$-dimensional manifold) and minimax optimal for \ref{enum:model_2} if $d \leq 1$ \cite{stone1982optimal}.

To the best of our knowledge, the literature does not provide algorithms
for estimating $f$ with the rate \eqref{eq:erm_rate_example} under the assumptions imposed in \ref{enum:model_1}
or \ref{enum:model_2}.
Focusing on \ref{enum:model_1}, a few special cases have been considered in the literature.
First, if $X$ is supported exactly on $\CM$,
classical methods such as k nearest neighbors, piecewise polynomials, or kernel methods achieve
the rate \eqref{eq:erm_rate_example}  \cite{kpotufe2011k, bickel2007local,ye2008learning}.
Second, if $\CM$ is a linear subspace, methods from sufficient dimension reduction literature combined with traditional estimators
achieve \eqref{eq:erm_rate_example} under certain reasonable assumptions \cite{ma2013review, li2018sufficient}.
Third, if $\dim(\CM) = 1$ and $g$ is strictly monotone along the manifold, \cite{kereta2019nonlinear}
achieves near-optimal rates in the case $\zeta \equiv 0$.
Still, none of these approaches achieves \eqref{eq:erm_rate_example} in the generality that is considered here,
which indicates a gap between the performance of `traditional estimators' and deep neural  networks.
We add though that computing the global minimizer \eqref{eq:erm_def}
within small (polynomial) runtime is not well-understood, because networks $\hat \Psi_N \in \CN_N$ are
underparametrized by the choice $P(\hat \Psi_N)\ll N$.

Finally, we note that
checking whether a function belongs to \ref{enum:model_1} or \ref{enum:model_2} is challenging in practice, because the input
$\{X_i : i \in [N]\}$ does not reveal the compositional nature of $x\mapsto f(x)$ by itself. Instead, the compositional nature
is only visible when jointly using $\{(X_i,Y_i) : i \in [N]\}$, for instance by inspecting
derivative tensors of the function $f$. As an example, Hessian matrices of functions that belong to  \ref{enum:model_1}
have at most $d$ nontrivial eigenvalues at any point $x \in \CA$ and the nontrivial eigenspace corresponds
to a subspace of the tangent space of $\CM$.
For functions of \ref{enum:model_2}, derivative tensors also tend to have a specific shape, whose precise
form depends on the distance $\normaldistnoargs$ and the parameter $M$.

\section{Preparatory material: a brief primer on ReLU calculus}
\label{sec:preliminaries}
ReLU calculus refers to a framework for developing ReLU network approximation
guarantees based on successively approximating increasingly complex building blocks. Corresponding results
have been developed in recent years \cite{yarotsky2017error,yarotsky2018optimal,petersen2018optimal,boelcskei2019optimal,grohs2019deep},
following the increased popularity of the ReLU activation in practice.
This section gives an overview of some of the results, which we use in the remainder.
Thoughout, deep ReLU networks are defined as stated in Definition \ref{def:relu_nets}.

The first  step towards developing approximation guarantees with ReLU nets
is to endow the space of ReLU nets with two basic operations, namely compositions and linear combinations.

\begin{lemma}[{Composition \cite[Lemma 2.5]{grohs2019deep}}]
\label{lem:cocatanation}
Let $\Phi_1 : \bbR^{N_0}\rightarrow \bbR^{N_{L_1}}$ and $\Phi_2 : \bbR^{N_{L_1}}\rightarrow \bbR^{N_{L_2}}$
be two ReLU nets. There exists a ReLU net $\Psi : \bbR^{N_0}\rightarrow \bbR^{N_{L_2}}$ with
$\Psi(x) = \Phi_2(\Phi_1(x))$ and $L(\Psi) = L(\Phi_1) + L(\Phi_2)$, $W(\Psi) = \max\{W(\Phi_1), W(\Phi_2), 2 N_{L_1}\}$,
$P(\Psi) = 2 (P(\Phi_1) + P(\Phi_2))$, and $\boundParams{\Psi} \leq \boundParams{\Phi_1} \vee \boundParams{\Phi_2}$.
\end{lemma}

\begin{lemma}[{Linear combination \cite[Lemma 2.7]{grohs2019deep}}]
\label{lem:relu_linear_combination}
Let $\{\Phi_i : i \in [N]\}$ be a set of ReLU networks with similar input dimension $N_{0}$.
There exist ReLU networks $\Psi_1$ and $\Psi_2$ that realize the maps
$\Psi_1(x) = (\alpha_1 \Phi_1(x),\ldots, \alpha_N \Phi_N(x))$ and $\Psi_2(x) = \sum_{i=1}^{N}\alpha_i \Phi_i(x)$.
For $j \in \{1,2\}$, they satisfy $L(\Psi_{j}) = \max_{i \in [N]}L(\Phi_i)$, $W(\Psi_j) \leq \sum_{i=1}^{N}\left(2 \vee W(\Phi_i) \right)$,
$P(\Psi_j) = \sum_{i=1}^{N}(P(\Phi_i) + W(\Phi_i) + 2(L-L(\Phi_i)) + 1)$,
and $\boundParams{\Psi_j} \leq \max \{1, \max_{i \in [N]} \boundParams{\Phi_i} \vee \alpha_i\}$.
\end{lemma}

Using compositions of ReLU nets, the next step is to approximate the square function $x\mapsto x^2$,
for instance by using the so-called `saw-tooth function' approximation \cite{yarotsky2017error}. Then, by using
the identity
$$
xy = \frac{1}{2}\left(x^2 + y^2 - (x-y)^2\right),
$$
one can establish approximation guarantees for arbitrary multiplication and for multivariate polynomials of arbitrary degree. We exemplarily report the results
of \cite{grohs2019deep} in the next lemma.

\begin{lemma}[{\cite[Proposition 3.2, 3.4 and 3.6]{grohs2019deep}}] Let $\varepsilon \in (0,1/2)$. \hfill
\label{lem:grohs_summary}

\noindent
1) There exists a network with $L(\Phi)\lesssim \log(1/\varepsilon)$,
$W(\Phi) = 3$, $P(\Phi) \lesssim \log(1/\varepsilon)$, and $\boundParams{\Phi}\leq 1$
such that
$
\sup_{x \in [0,1]}\SN{\Phi(x) - x^2}\leq \varepsilon.
$

\noindent
2) Let $R \geq 1$. There exists a network $\Phi$ with $L(\Phi)\lesssim \log(R/\varepsilon)$,
$W(\Phi)\leq 5$, $P(\Phi)\lesssim \log(R/\varepsilon)$ and $B(\Phi)\leq 1$ so that
$
\sup_{(x,y) \in [-R, R]^2}\SN{\Phi(x,y) - xy}\leq \varepsilon.
$

\noindent
3) Let $m \in \bbN$, $a \in \bbR^{m+1}$, $R \geq 1$. There exists a network $\Phi$ with $L(\Phi)\lesssim m\log(1/\varepsilon)+m^2 \log(R) + m\log(\lceil \N{a}_{\infty}\rceil)$,
$W(\Phi)\leq 9$, $P(\Phi)\lesssim L(\Phi)$, $B(\Phi)\leq 1$, and
$
\sup_{x \in [-R, R]}\SN{\Phi(x) - \sum_{i=0}^{m}a_ix^i}\leq \varepsilon.
$
\end{lemma}

\begin{table}[t]
\begin{center}
\scriptsize
\begin{tabular}{@{}cccccccc@{}}
      map & metric & $L(\Phi)$ & $W(\Phi)$ & $P(\Phi)$ & $B(\Phi)$  & Reference  \\ \toprule
$x\mapsto \N{x}_p^p$ & $L_{\infty}([-R,R]^D)$ & $\CO(p^2\log(\lceil R\rceil D/\varepsilon))$ & $9D$ & $\CO(DL(\Phi))$ & $1$  & Lem. \ref{lem:relu_lp_norm}\\
$(x,t)\mapsto tx$ & $L_{\infty}([-R,R]^{D+1})$ & $\CO(\log(R^2/\varepsilon))$ & $5D$ & $\CO(DL(\Phi))$ & $1$ & Lem. \ref{lem:multiplication_network}\\
$t \mapsto 1/t$ & $L_{\infty}([R^{-1},R])$ & $\CO(R^4\log^2(R/\varepsilon))$ & $9$ & $\CO(L(\Phi))$ & $1$   & Lem. \ref{lem:division}\\
$x \mapsto x/\N{x}_1$ & $L_{\infty}(\{x : R^{-1}\leq \N{x}_1 \leq R\})$ & $\CO(R^4\log^2(R/\varepsilon))$ & $\CO(D)$ & $\CO(DL(\Phi))$ & $1$ &  Lem. \ref{lem:l1_normalization}\\
$x \mapsto \min_i x_i$ & Exact on $\bbR^D$ & $2\lceil\log_2(D)\rceil$ & $3\lceil D/2\rceil$ & $11D\lceil\log_2(D)\rceil$ & $1$  & Lem. \ref{lem:relu_calculus_minimum}
\\\bottomrule
\end{tabular}
\end{center}
\caption{Basic ReLU calculus results that are relevant to the manuscript. We use $x \in \bbR^D$ for vectors and $t \in \bbR$ for scalars. The approximation accuracy is $\varepsilon$
in the respective metric and $\CO(\cdot)$ means as $\varepsilon \rightarrow 0$. $L$, $W$, $P$, and $B$ denote bounds on depth, width, number of parameters,
and coefficient size of the network respectively.}
\label{tab:relu_calculus}
\end{table}

A natural next step is to study the approximation of
functions with a certain degree of regularity. This can be done for instance by using local Taylor expansions
 and by approximating Taylor polynomials and indicator functions
through deep networks. As a result, ReLU nets are able to uniformly approximate functions
with a certain degree of regularity  with an optimal number of
function queries and nonzero parameters, see e.g. \cite{yarotsky2017error,yarotsky2018optimal,
schmidt2017nonparametric}. As an example (and since it suffices for our purposes), we
present a simplified version of \cite[Theorem 5]{schmidt2017nonparametric} for $\alpha$-H\"older univariate functions.

\begin{theorem}[{Simplified version of \cite[Theorem 5]{schmidt2017nonparametric}}]
\label{thm:smooth_function_approx}
Let $L \geq 1$, $\alpha \in (0,1]$ and consider
$f : [0,1]\rightarrow \bbR$ with
$
\SN{f(t) - f(s)} \leq L\SN{t-s}^{\alpha}$ for all $t,s \in [0,1].
$
For any $\varepsilon  > 0$ there exists a ReLU network $\Phi$ that uses
$n \lesssim \varepsilon^{-1}$ point queries of $f$ and has complexity bounded by
$L(\Phi)\lesssim \log(1/\varepsilon)$,
$W(\Phi)\lesssim 1/\varepsilon$, $P(\Phi) \lesssim \log(1/\varepsilon) 1/\varepsilon$,
$\boundParams{\Phi} \leq 1$ such that
\begin{align*}
\sup_{x \in [0,1]^k}\SN{f(x) -  \Phi(x)} \leq L \varepsilon^{\alpha}.
\end{align*}
\end{theorem}
\begin{proof}
With $\alpha \in (0,1]$, \cite[Theorem 5]{schmidt2017nonparametric} gives (using the same notation as in the reference)
\begin{align*}
\sup_{x \in [0,1]^k}\SN{f(x) -  \Phi(x)}\lesssim LN2^{-m} + LN^{-\alpha},
\end{align*}
where $N$ and $m$ effectively describe width and depth of the approximating network $\Phi$. By choosing $N \asymp 1/\varepsilon$
and $m \asymp \log_2(1/\varepsilon^{(1+\alpha)} ))$ with suitable universal constants both summands are bounded by $L\varepsilon^{\alpha}/2$, giving
the asserted approximation guarantee. The required network size can be read of from \cite[Theorem 5]{schmidt2017nonparametric}
by inserting $N$ and $m$. For counting the number of required queries of $f$, we note that $\Phi$ approximates a piecewise constant
approximation of $f$ based on $\CO(\varepsilon^{-1})$ subintervals.
\end{proof}

The aforementioned results present a small subset of existing approximation results for ReLU nets
and give an idea how we can gradually approximate maps of increasing complexity. To faciliate the proofs for our results
in the next two sections, we require some additional elementary approximations. These are listed in Table \ref{tab:relu_calculus},
with proofs deferred to Section \ref{subsec:relu_calculus} in the Appendix.

\section{Proof of Theorem \ref{thm:approximating_function_projection}}
\label{subsec:proofs_approximation_projection_new}
We first give a proof sketch
that outlines the strategy and additionally highlights the main challenges compared to the
previously studied case $\CA = \CM$ and $\pi_{\CM}=\Id$. Afterwards we present the proof details.
Throughout we let $C_d, C_\CM$ and $C_q$ be the constants defined in
Theorem \ref{thm:approximating_function_projection}.

\subsection{Proof sketch and comparison with the case $\CA = \CM$}
\label{subsec:proof_details_sketch}
Our proof strategy shares some similarities with existing proof strategies for the case
$\CA = \CM$, see for instance \cite{shaham2018provable,schmidt2019deep,nakada2019adaptive},
but also differs in some aspects due to additional complications arising from the high-dimensional
approximation domain. In both cases, we can start with a maximal separated $\delta$-net $\{z_1,\ldots,z_K\}$ of $\CM$ (see Definition \ref{def:geoodesic_covering_number}), which has
cardinality bounded by $K\approx C_{\CM}\delta^{-d}$ according to Lemma \ref{lem:auxiliary_results_diff_geom} below.
Then, by defining $U_i$ as geodesic balls $U_i := \{z \in \CM : d_{\CM}(z,z_i)\leq\delta\}$, the subsets
$U_1,\ldots,U_K$ cover the manifold $\CM$ and the preimages  $\pi_{\CM}^{-1}(U_1),\ldots,\pi_{\CM}^{-1}(U_K)$ cover the approximation domain $\CA \subseteq \CM(q)$.
Hence, for any partition of unity $\eta_1,\ldots,\eta_K$ subject to $\pi_{\CM}^{-1}(U_1),\ldots,\pi_{\CM}^{-1}(U_K)$,
we can express $f$ by $f(x) = \sum_{i}f(x)\eta_i(x)$.

Let us now denote the orthoprojector onto the tangent space at $z_i \in \CM$ by $A_i$.
If we are in the case $\CA = \CM$, we naturally have $U_i = \pi_{\CM}^{-1}(U_i) \cap \CA$ and the sets $U_i$ are isomorphic
to $A_i(U_i)$ (provided $\delta < \reach/2$, i.e., the covering of $\CM$ is sufficiently fine \cite{shaham2018provable,schmidt2019deep}). Therefore, approximating
$f\eta_i$ over $U_i$ is morally like approximating a function on a compact
subset of $\bbR^d$ and we can apply results from
\cite{mhaskar1993approximation,mhaskar1996neural,yarotsky2017error} to achieve approximation
guarantees that depend exponentially on $d$ instead of $D$. By linear combination of
the resulting $C_{\CM}\delta^{-d}$ approximants, we then obtain an approximation to $f$.

In the case $\CM \subset\CA \subseteq \CM(q)$ the aforementioned strategy unfortunately can not be used, because
each $\eta_i$ in the partition of unity is supported on a compact subset of $\bbR^D$, which is not isomorphic to a compact set in $\bbR^d$. Hence,
naively using results from \cite{mhaskar1993approximation,mhaskar1996neural,yarotsky2017error} to approximate an arbitrary
partition of unity $\eta_1,\ldots,\eta_K$ subject to $\pi_{\CM}^{-1}(U_1),\ldots,\pi_{\CM}^{-1}(U_K)$ incurs the curse of dimensionality.

Instead, we will use a finer covering of $\CM$ at the scale
$\delta \approx \varepsilon$ (as opposed to $\delta \approx \reach$ in the case $\CA = \CM$)
and employ the piecewise constant approximation
$
f(x) = \sum_{i}f(x)\eta_i(x) \approx \sum_{i}g(z_i)\eta_i(x).
$
If $\eta_1,\ldots,\eta_K$ form a partition of unity with the localization property
\begin{align}
\label{eq:localization_property_1}
\sup_{x \in \CM(q) : \eta_i(x) \neq 0} d_{\CM(q)}(x, z_i) \lesssim \varepsilon,
\end{align}
the piecewise constant approximation $f(x)\approx \sum_{i}g(z_i)\eta_i(x)$ is accurate up to $\CO(\varepsilon^{\alpha})$
for $\alpha$-H\"older $g$. We note however
that we have to approximate $K \approx C_{\CM}\varepsilon^{-d}$ functions $\eta_1,\ldots,\eta_K$
by deep networks, which means that we can allocate at most $\CO(\textrm{polylog}(\varepsilon^{-1}))$ nonzero parameters
for each individual approximation  to match the overall result achieved in Theorem \ref{thm:approximating_function_projection}. Thus,
$\eta_i$'s have to satisfy \eqref{eq:localization_property_1}, while also being approximable
by relatively small networks.

Designing such $\eta_i$'s is main difficulty of the proof. We first derive an auxiliary result
to locally approximate the extended geodesic metric $d_{\CM(q)}(x,z_i) = d_{\CM}(\pi_{\CM}(x),z_i)$ around $z_i$ by basic features of the input vector $x$.
Namely, Proposition \ref{prop:metric_equivalency} shows that, for any $p \in [q,1)$, we have the local metric equivalence
\begin{align}
\label{eq:metric_equivalency_proof_sketch}
\N{A(z_i)^\top(x-z_i)}_2 \lesssim d_{\CM(q)}(x, z_i) \lesssim \frac{1}{1-p}\N{A(z_i)^\top(x-z_i)}_2
\end{align}
for every $x$ contained in $B_{p\lreach{z_i}}(z_i)$ and with $\N{A(z_i)^\top(x-z_i)}_2\lesssim (1-p)\reach$.
Intuitively, a point $x$ satisfies $x \in B_{p\lreach{z_i}}(z_i)$ and $\N{A(z_i)^\top(x-z_i)}_2\lesssim (1-p)\reach$ if it is contained in an $L_2$-ball around $z_i$
that extends up to $p\lreach{z_i}$ in normal direction and $(1-p)\reach$ in tangential direction.
Equation \eqref{eq:metric_equivalency_proof_sketch} implies that
we can approximate $d_{\CM(q)}(x,z_i)$ on such balls using $\textN{A(z_i)^\top(x-z_i)}_2$.

Crucially, $\N{x-z_i}_2$ and $\textN{A(z_i)^\top(x-z_i)}_2$ are simple features
 of the input $x$, because (after taking squares) they are composed of a linear
transformation followed by a polynomial of degree $2$ of the input $x$. Hence, we
can approximate these features efficently using deep networks and construct a partition
of unity function accordingly. The precise construction reads
\begin{align*}
\tilde \eta_i(x) := \relu{1 - \left(\frac{\N{x-z_i}_2}{p\lreach{z_i}}\right)^2 - \left(\frac{\N{A(z_i)^\top(x-z_i)}_2}{h\varepsilon}\right)^2}\quad \textrm{ and }\quad \eta_i(x) := \frac{\tilde \eta_i(x)}{\N{\tilde \eta(x)}_1},
\end{align*}
where $h$ is a bandwidth parameter that is suitably chosen as a function of $q$ and $\reach$.
As shown in Proposition \ref{prop:construction_partition_of_unity} and Lemma \ref{lem:existence_of_eta_network},
$\eta_i$ satisfies \eqref{eq:localization_property_1} and can be approximated to accuracy $\varepsilon$ by a ReLU network $\Theta_i$ with
$\CO(\textrm{polylog}(\varepsilon^{-1}))$ nonzero parameters.

\begin{figure}
\centering
\includegraphics[scale = 0.75]{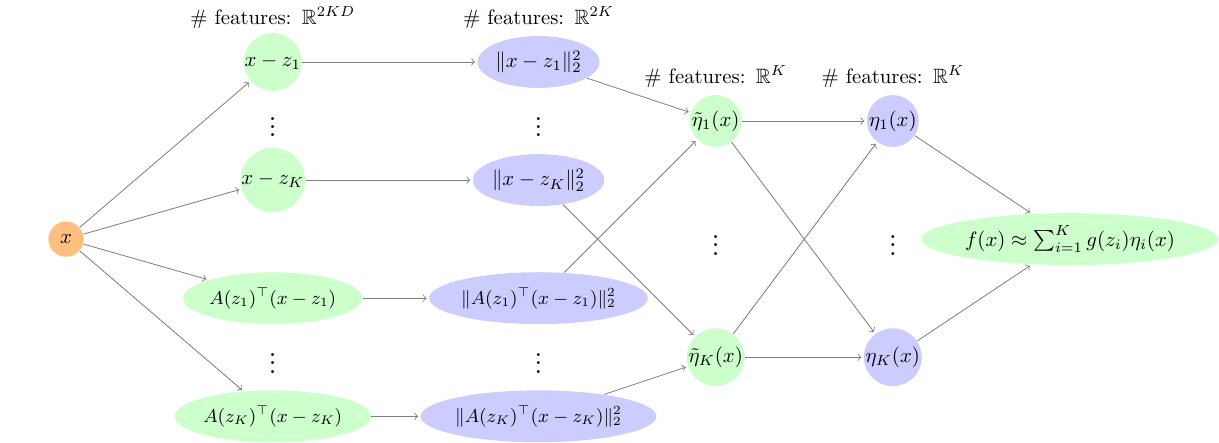}
\caption{Schematic ReLU construction used to approximate \ref{enum:model_1}. At each node we illustrate the feature of $x$ that
is being approximated by the network. Green nodes can be exactly realized (assuming the previous layer is exact) with finite width layers, whereas
blue nodes are approximated to accuracy $\CO(\varepsilon)$ using $\CO(\textrm{polylog}(\varepsilon^{-1}))$ layers.}
\label{fig:network_projection}
\end{figure}
Finally, after recalling that linear combinations of ReLU networks are still
ReLU networks, we approximate $f$ by
\begin{align}
\label{eq:network_construction_proof_sketch}
\Phi(x) = \sum_{i=1}^{K}g(z_i)\Theta_i(x).
\end{align}
A schematic illustration
of the complete approximating network is depicted in Figure \ref{fig:network_projection}.

\subsection{Proof details}
\label{subsec:projection_proof_details}
Let us first collect some elementary
facts from differential geometry that are required in the following.
\begin{lemma}
\label{lem:auxiliary_results_diff_geom}
Let $\CM$ be a $d$-dimensional compact connected Riemannian manifold embedded in $[0,1]^D$
with reach $\reach > 0$, Lebesgue volume $\textrm{Vol}(\CM)$,
and endowed with the Riemannian metric induced by $\bbR^D$. Let $v,z \in \CM$.

\noindent
1) If $\N{v-z}_2\leq \reach/2$ then $d_{\CM}(v,z) \leq \reach(1-\sqrt{1-2\N{v-z}_2/\reach})$.

\noindent
2) For any $r \in (0,\reach/2)$  we have
$\textrm{Vol}(B_{\CM,r}(v)) \leq C_d(\reach/(\reach-2r))^dr^d$.

\noindent
3) The tangent space orthoprojectors $A(v)A(v)^\top \in \bbR^{D\times D}$ satisfy
perturbation bounds
\begin{align}
\label{eq:tangent_perturbation}
\N{A(v)A(v)^\top - A(z)A(z)^\top}_2 \leq \frac{1}{\reach}d_{\CM}(v,z).
\end{align}
\noindent
4) The local reach as defined in \eqref{eq:local_reach} satisfies the perturbation bound
\begin{align}
\label{eq:lipschitz_continuity_lreach}
\SN{\lreach {v} -\lreach{z}} \leq \N{v-z}_2 \leq d_{\CM}(v,z).
\end{align}
5) We have
$\CP(\delta,\CM,d_{\CM}) \leq 3^d \textrm{Vol}(\CM) d^{d/2}\delta^{-d}$ for any $\delta \in (0,\frac{1}{2}\reach)$.

\noindent
6) Let $\CZ$ be a maximal $\delta$-separated set of $\CM$ with respect to the geodesic metric. For
any $p$ with $p\delta \in (0,\reach/4)$ we have $\SN{\CZ \cap B_{\CM,p\delta}(v)} \leq C_d (5p\sqrt{d})^d$.
\end{lemma}
\begin{proof}
Property 1) can be found in \cite[Lemma 3]{genovese2012minimax} and 2) is derived in \cite[Proposition 1.1]{chazal2013upper}.
3) is similar to \cite[Corollary 3]{boissonnat2019reach}, after noticing that
\begin{align*}
\N{A(v)A(v)^\top - A(z)A(z)^\top}_2 = \sin \angle\left(A(v),  A(z)\right) \leq 2\sin\left( \frac{ \angle\left(A(v),  A(z)\right)}{2}\right),
\end{align*}
where $\angle\left(A(v),  A(z)\right)$ denotes the maximum principal angle between subspaces $\Im(A(v))$ and $\Im(A(z))$.
For 4) we assume without loss of generality $\lreach{v} \geq \lreach{z}$. Then the result follows from
\begin{align*}
\lreach{v}-\lreach{z} &= \dist{v}{\medial}-\dist{z}{\medial} \\
&\leq \N{v-z}_2 + \dist{z}{\medial} - \dist{z}{\medial} = \N{v-z}_2 \leq d_{\CM}(v,z).
\end{align*}
Property 5) can be found in \cite{niyogi2008finding,baraniuk2009random}.
For 6) we first note that $\CZ \cap B_{\CM,p\delta}(v)$ is still a $\delta$-separated set of the geodesic ball $B_{\CM,p\delta}(v)$,
which implies $\SN{\CZ \cap B_{\CM,p\delta}(v)} \leq \CP(\delta,B_{\CM,p\delta}(v), d_{\CM})$. Since the reach
of the geodesic ball $B_{\CM,p\delta}(v)$ is also bounded by $\reach$, we can apply
Property 2) and Property 5) to get
\[
\CP(\delta, B_{\CM,p\delta}(v) , d_{\CM}) \leq  \frac{3^d \textrm{Vol}(B_{\CM,p\delta}(v)) d^{\frac{d}{2}}}{\delta^{d}}
\leq \frac{3^d 2^d C_d p^d\delta^d d^{\frac{d}{2}}}{\delta^{d}} = C_d(5p\sqrt{d})^d.
\]
\end{proof}

The first step to prove Theorem \ref{thm:approximating_function_projection} rigorously establishes the local metric equivalence \eqref{eq:metric_equivalency_proof_sketch}
between the geodesic metric $d_{\CM}(\pi_{\CM}(x),z)$ and $\N{A(z)^\top(x - z)}_2$.

\begin{proposition}
\label{prop:metric_equivalency}
Let $\CM$ be a connected compact $d$-dimensional Riemannian submanifold of $\bbR^D$
and let $q \in [0,1)$. For $x \in \CM(q)$ with $v = \pi_{\CM}(x)$ and arbitrary $z \in \CM$ we have
\begin{equation}
\label{eq:metric_upper_bound}
\N{A(z)^\top(x - z)}_2 \leq \left(1 + \frac{\dist{x}{\CM}}{\reach \vee (\lreach{v}-d_{\CM}(v,z))}\right)d_{\CM}(z, v).
\end{equation}
Let now $p \in [q,1)$ arbitrary. Then for
$x \in B_{p\lreach{z}}(z)$ with $\N{A(z)^\top(x - z)}_2 < \frac{1-p}{3}\reach$, we have
\begin{align}
\label{eq:metric_lower_bound_2}
d_{\CM}(z,v) &\leq \frac{3}{1-p}\N{A(z)^\top(x - z)}_2.
\end{align}
\end{proposition}
\noindent

\begin{proof}
Throughout the proof we denote $P(z) = A(z)A(z)^\top$ as the orthoprojector onto
the tangent space of $\CM$ at $z \in \CM$.
For \eqref{eq:metric_upper_bound} we use $P(v)(x-v) = 0$
from part 1) of Lemma \ref{lem:combined_unique_projection_lipschitz_property},
$\N{z-v}_2 \leq d_{\CM}(z,v)$, and the tangent perturbation bound \eqref{eq:tangent_perturbation}
applied to the geodesic path $\gamma_{z\rightarrow v}$ from $z$ to $v$ with
reach bound $\tau_{\gamma_{z\rightarrow v}} = \inf_{y \in \Im(\gamma_{z\rightarrow v})}\lreach{y}$ to compute
\begin{align*}
\N{P(z)(x-z)}_2 &\leq \N{P(z)(v - z)}_2 + \N{P(z)(x - v)}_2 \leq d_{\CM}(v,z) + \N{P(z) - P(v)}_2 \N{x - v}_2\\
&\leq d_{\CM}(v,z) + \frac{\dist{x}{\CM}}{\tau_{\gamma_{z\rightarrow v}}} d_{\CM}(v,z).
\end{align*}
Furthermore, by the $1$-Lipschitz property of the local reach, see \eqref{eq:lipschitz_continuity_lreach}, we have
\begin{align*}
\tau_{\gamma_{z\rightarrow v}} = \inf_{y \in \Im(\gamma_{z\rightarrow v})}\lreach{y} \geq \lreach{v} - \sup_{y \in \Im(\gamma_{z\rightarrow v})}\SN{\lreach{y}-\lreach{v}}
\geq \lreach{v} - d_{\CM}(v,z).
\end{align*}
Since the global bound $\tau_{\gamma_{z\rightarrow v}} \geq \reach$ holds due to $\Im(\gamma_{z\rightarrow v}) \subset \CM$, we obtain
\begin{align*}
\N{P(z)(x-z)}_2 & \leq \left(1 + \frac{\dist{x}{\CM}}{\reach \vee (\lreach{v}-d_{\CM}(v,z))}\right)d_{\CM}(v,z).
\end{align*}

\noindent
For the opposite direction \eqref{eq:metric_lower_bound_2} we
let $\omega := \N{P(z)(x-z)}_2$ and $\tilde x := z + Q(z)(x-z)$,
where $Q(z) := \Id - P(z)$. By construction we have $P(z)(\tilde x - z) = 0$ and
\begin{align*}
\N{x-\tilde x}_2 = \N{x-z - Q(z)(x-z)}_2 = \N{P(z)(x-z)}_2 = \omega.
\end{align*}
Furthermore, since $x \in B_{p\tau_{\CM}(z)}(z)$, $\omega < \frac{1-p}{3}\reach$, and $\reach\leq \lreach{z}$, we can bound
$$
\N{\tilde x - z}_2 \leq \N{x - z}_2 + \N{x-\tilde x}_2 \leq p\tau_{\CM}(z) + \omega < p\tau_{\CM}(z) + \frac{1-p}{3}\tau_{\CM} < \tilde p\tau_{\CM}(z),
$$
for $\tilde p = \frac{1+2p}{3} < 1$. We thus have the decomposition $\tilde x = z + (\tilde x - z)$ for $z \in \CM$, $\tilde x - z \perp \Im(P(z))$, and
$\N{\tilde x - z}_2 < \tilde p \tau_{\CM}(z)$. Part 1) in Lemma \ref{lem:combined_unique_projection_lipschitz_property} implies $z = \pi_{\CM}(\tilde x)$ and $\tilde x \in \CM(\tilde p)$.
Using the Lipschitz property of $\pi_{\CM}$ in part 2) of Lemma \ref{lem:combined_unique_projection_lipschitz_property} and
$x \in \CM(q) \subset \CM(\tilde p)$, $\tilde x \in \CM(\tilde p)$, we get
\begin{align*}
\N{v-z}_2 &= \N{\pi_{\CM}(x) - \pi_{\CM}(\tilde x)}_2 \leq \frac{1}{1-\tilde p}\N{x-\tilde x}_2 = \frac{3}{2(1-p)}\omega.
\end{align*}
We further not that $\frac{3}{2(1-p)}\omega  < \frac{1}{2}\reach$,
so that we can apply part 1) of Lemma \ref{lem:auxiliary_results_diff_geom} to get
\[
d_{\CM}(v,z) \leq \reach - \reach \sqrt{1-\frac{2\N{v-z}_2}{\reach}} \leq \N{v-z}_2 + \frac{2\N{v-z}_2^2}{\reach} \leq 2 \N{v-z}_2.
\]
\end{proof}

We now introduce the partition of unity functions $\eta_1,\ldots,\eta_K$
and show that they satisfy the desired localization property.

\begin{proposition}
\label{prop:construction_partition_of_unity}
Consider a connected compact $d$-dimensional Riemannian submanifold $\CM \subseteq \bbR^D$
and let $q \in [0,1)$. Let $\CZ = \{z_1,\ldots,z_{\SN{\CZ}}\} \subset \CM$ be a maximal $\delta$-separated
set of $\CM$ with respect to $d_{\CM}$.
Define bandwidth parameters $p:=\frac{1}{2}(1+q)$ and  $h := \frac{6}{1-qp^{-1}}$ and  functions $\tilde \eta, \eta : \CM(q) \rightarrow \bbR^{\SN{\CZ}}$
componentwise by
\begin{align}
\label{eq:definition_pou_function}
\tilde \eta_i(x) = \relu{1 - \left(\frac{\N{x-z_i}_2}{p\lreach{z_i}}\right)^2 - \left(\frac{\N{A(z_i)^\top(x-z_i)}_2}{h\delta}\right)^2}\quad \textrm{ and }\quad \eta_i(x) = \frac{\tilde \eta_i(x)}{\N{\tilde \eta(x)}_1}.
\end{align}
There exists a universal constant $C$ such that if $ \delta \in (0,C(1-q)^2\reach)$ we have
\begin{align}
\label{eq:POU_guarantee}
\sup_{x \in \CM(q) : \eta_i(x) \neq 0}d_{\CM(q)}(x, z_i) &\lesssim \frac{\delta}{(1-q)^2},\\
\label{eq:POU_norm_bound}
(1-q)\lesssim\N{\tilde \eta(x)}_1 &\lesssim C_q.
\end{align}
\end{proposition}
\begin{proof}
Denote $v = \pi_{\CM}(x)$. We will a few times require in the following the bandwidth ratio
\begin{align*}
\frac{3h}{1-p} = \frac{36(q+1)}{(1-q)^2} \in \left[\frac{36}{(1-q)^2}, \frac{72}{(1-q)^2}\right).
\end{align*}
By construction $\eta_i(x) \neq 0$ implies $x \in B_{p\lreach{z_i}}(z_i)$ and $\textN{A(z_i)^\top(x-z_i)}_2 < h\delta$.
Thus, as soon as  $\delta < \frac{1-p}{3h}\reach$, which is implied by $\delta < \frac{1}{36}(1-q)^2\reach$,
 we have$\N{A(z_i)^\top(x-z_i)}_2 \leq \frac{1-p}{3}\reach$. Applying
Proposition \ref{prop:metric_equivalency} gives \eqref{eq:POU_guarantee} by
$$
d_{\CM(q)}(x, z_i) = d_{\CM}(v,z_i)\leq \frac{3h}{1-p}\delta \leq \frac{72}{(1-q)^2}\delta.
$$

\noindent
We now concentrate on the lower bound in \eqref{eq:POU_norm_bound}. Denote $j \in \argmin_{i \in \SN{\CZ}}d_{\CM(q)}(x,z_i)$. Since $\CZ$ is a maximal $\delta$-separated
set of $\CM$, we have $d_{\CM(q)}(x, z_j)\leq \delta$. Eqn. \eqref{eq:metric_upper_bound} in Proposition \ref{prop:metric_equivalency}
implies
\begin{align*}
\N{A(z_j)^\top(x-z_j)}_2  &\leq \left(1+\frac{\dist{x}{\CM}}{\lreach{v} - \delta}\right)\delta \leq \left(1+\frac{q\lreach{v}}{\lreach{v} - \delta}\right)\delta\\
&=\left(1+q\frac{1}{1 -\frac{\delta}{\lreach{v}}}\right)\delta \leq (1+2q)\delta \leq 3\delta,\qquad \textrm{provided } \delta < \frac{1}{2}\reach.
\end{align*}
Using the triangle inequality to get  $\N{x-z_j}_2 \leq \delta + \N{x-v}_2$
and the $1$-Lipschitz continuity of $\lreach{\cdot}$ in \eqref{eq:lipschitz_continuity_lreach}
to further bound $\N{x-v}_2 \leq q\lreach{v} \leq q(\delta + \lreach{z_j})$,
it follows that
\begin{align*}
\frac{\N{x-z_j}_2}{p\lreach{z_j}}\leq \frac{\delta + q\delta + q\lreach{z_j}}{p\lreach{z_j}}\leq \frac{q}{p} + \frac{1+q}{p\reach}\delta \leq \frac{q}{p} + \frac{4}{\reach}\delta
\end{align*}
Inserting the definition of the bandwidth parameter $h$, we thus obtain
\begin{equation}
\label{eq:aux_squaring_comment}
\begin{aligned}
1\! -\! \frac{\N{x-z_j}_2}{p\lreach{z_j}}\! -\! \frac{\N{A(z_j)^\top(x-z_j)}_2}{h\delta} &\geq 1\!-\! \frac{q}{p}-\!\frac{4}{\reach}\delta \!
- \frac{3}{h} \geq \frac{1}{2}\left(\!1\!-\!\frac{q}{p}\right) - \frac{4}{\reach}\delta.
\end{aligned}
\end{equation}
This is bounded from below by $\frac{1}{4}(1-qp^{-1})$ as soon as
\begin{align*}
\delta < \frac{\reach}{16}\left(1-\frac{q}{p}\right) = \frac{\reach}{16}\frac{1-q}{1+q}, \qquad \textrm{which is implied by}\qquad \delta<\frac{(1-q)\reach}{16}.
\end{align*}
Since squaring one of the subtracted
terms in \eqref{eq:aux_squaring_comment} reduces their size, we get the lower bound
$\N{\tilde \eta(x)}_1 \geq \tilde \eta_i(x) \geq \frac{1}{4}(1-pq^{-1})\geq 1/8(1-q)$.

\noindent
For the upper bound on $\N{\tilde \eta(x)}_1$ we notice that $\eta_i(x) \neq 0$ implies by Proposition \ref{prop:metric_equivalency}
\begin{align*}
h\delta > \N{A(z_i)^\top(x-z_i)}_2 \geq \frac{1-p}{3} d_{\CM(q)}(z_i, x) \qquad \textrm{ provided }\qquad \delta < \frac{(1-q)^2}{36}\reach.
\end{align*}
Thus, $\tilde \eta_i(x) \neq 0$ implies $d_{\CM(q)}(x,z_i) \leq 3h(1-p)^{-1}\delta$, i.e., all $z_i$'s contributing to
$\N{\tilde \eta}_1$ are contained within a geodesic ball of radius $3h(1-p)^{-1}\delta$ around $v$.
As soon as $ \frac{3h}{1-p}\delta < \frac{1}{4}\reach$, which is implied by $\delta < 288(1-q)^2\reach$,
we can then use part 5) of Lemma \ref{lem:auxiliary_results_diff_geom} to bound
$$
\SN{\CZ \cap B_{\CM, \frac{3h}{1-p}\delta}(v)} \leq C_d \left(5\frac{3h}{1-p}\sqrt{d}\right)^d
\leq C_d \left(5\frac{72}{(1-q)^2}\sqrt{d}\right)^d \lesssim C_q.
$$
Since each $\tilde \eta_i(x)$ is individually bounded by $1$, the upper bound on
$\N{\tilde \eta(x)}_1$ in \eqref{eq:POU_norm_bound} follows.
\end{proof}

We next show that $\eta$ can be uniformly approximated by a ReLU net of small complexity.

\begin{lemma}
\label{lem:existence_of_eta_network}
Assume the setting of Proposition \ref{prop:construction_partition_of_unity} with
$\delta < \reach/2$ and $\CM(q) \subseteq [0,1]^D$.
For all $\varepsilon \in (0,1)$ there exists a ReLU-net $\Phi$ with complexity bounded as in
\eqref{eq:bounds_architecture_1} such that
\begin{align}
\sup_{x\in \CM(q)}\N{\eta(x)-\Phi(x)}_1 \leq \varepsilon.
\end{align}
\end{lemma}
\begin{proof}
Recall that $\CZ = \{z_1,\ldots,z_{\SN{\CZ}}\} \subset \CM$ is a maximal $\delta$-separated
set of $\CM$ with respect to $d_{\CM}$ and that we have
$C_{q}^{-1} \lesssim \N{\tilde \eta(x)}_1 \lesssim C_{q}$ (see right hand side in \eqref{eq:POU_norm_bound}).
The proof is split into two parts. First, we describe
how to approximate $\tilde \eta_i$ for some $i \in [\SN{\CZ}]$, and afterwards
we describe how to combine the networks to approximate $\eta : \bbR^D \rightarrow \bbR^{\SN{\CZ}}$.

\noindent
\textbf{1. Approximating $\tilde \eta_i$:}
Let $\Theta$ be a ReLU net that approximates $\N{\cdot}_2^2$ over $[-1,1]^D$ to accuracy
$\tilde \varepsilon > 0$ (existence is proven in Lemma \ref{lem:relu_lp_norm}). Furthermore, let $\Psi_i$ realize $x \mapsto x-z_i$,
and $\Gamma_i$ realize $x \mapsto A(z_i)^\top (x-z_i)$. For
bandwidth parameters $p$ and $h$ as in Proposition \ref{prop:construction_partition_of_unity},
we then define a ReLU network
$$
\tilde \Phi_i(x) := \relu{1-\frac{\Theta(\Psi_i(x))}{(p\lreach{z_i})^2} - \frac{\Theta(\Gamma_i(x))}{(h\delta)^2}}.
$$
Comparing $\tilde \Phi_i$ with $\tilde \eta_i$ we obtain by $1$-Lipschitzness of the ReLU and the triangle inequality
\begin{equation}
\label{eq:approximation_tilde_eta}
\begin{aligned}
\sup_{x \in \CM(q)}\SN{\tilde \Phi_i(x) - \tilde \eta_i(x)} &\leq \SN{\frac{\Theta(\Psi_i(x))-\N{x-z_i}_2^2}{(p\lreach{z_i})^2}} + \SN{\frac{\Theta(\Gamma_i(x))-\N{A(z_i)^\top(x-z_i)}_2^2}{(h\delta)^2}}\\
&\leq \frac{\tilde \varepsilon}{(p\lreach{z_i})^2} + \frac{\tilde \varepsilon}{(h\delta)^2} \leq \left(\frac{1}{(p\reach)^2} + \frac{1}{(h\delta)^2}\right)
\leq \frac{5\tilde \varepsilon }{(\reach \delta)^2}
\end{aligned}
\end{equation}
where we used $x-z_i \in [-1,1]^D$ since $x,z_i \in [0,1]^D$, $p \geq 1/2$, and $h > 1$.
To compute the complexity of $\tilde \Phi_i$
we apply the rules of ReLU composition and linear combination in Lemma \ref{lem:cocatanation} and \ref{lem:relu_linear_combination},
and the complexity bounds in Lemma \ref{lem:relu_lp_norm}.
We have $L(\Theta \circ \Psi_i) \leq L(\Theta) + L(\Psi_i) \lesssim \log(D/\tilde \varepsilon)$,
$W(\Theta \circ \Psi_i) \lesssim D$, $P(\theta\circ\Psi_i) \lesssim D\log(D/\tilde \varepsilon)$,
and $B(\Theta \circ \Psi_i) \leq 1$,
and the same bounds hold for $\Theta \circ \Gamma_i$. Thus, by the rules of ReLU linear combination
in Lemma \ref{lem:relu_linear_combination} (the additional ReLU activation in the last layer does matters for the absolute bounds) we have
\begin{align*}
L(\tilde \Phi_i)\lesssim \log(D/\varepsilon^{-1}),\quad  W(\tilde \Phi_i) \lesssim D,\quad  P(\tilde \Phi_i)\lesssim D\log(D\tilde\varepsilon^{-1}),
\quad B(\tilde \Phi_i) \leq 1/(p\reach)^2 \vee 1/(h\delta)^2.
\end{align*}

\noindent
\textbf{2. Approximating $\eta$:}
Define now $\tilde \Phi(x) = (\tilde \Phi_1(x),\ldots,\tilde \Phi_{\SN{\CZ}}(x))$. Using \eqref{eq:approximation_tilde_eta} we note that
\begin{align}
\label{eq:aux_projection_2}
\SN{\textN{\tilde \Phi(x)}_{1} - \textN{\tilde \eta(x)}_1} \leq \textN{\tilde \Phi(x) - \tilde \eta(x)}_1
\leq \sum_{i=1}^{\SN{\CZ}}\textSN{\tilde \Phi_i(x) - \eta_i(x)} \leq \frac{5\SN{\CZ}\tilde \varepsilon }{(\reach \delta)^2}
\end{align}
Thus, with $C_{q}^{-1} \leq \N{\tilde \eta(x)}_1\leq C_{q}$
we get  $1/2 C_{q}^{-1} \leq \textN{\tilde \Phi(x)}_1\leq 2C_{q}$ for  $\tilde \varepsilon \leq (\reach \delta)^2/(10C_{q} \SN{\CZ})$. Now, let $\Lambda$ be a network
that approximates $\ell_1$-normalization up to $\varepsilon/2$ for inputs $u$ with
$(2C_{q})^{-1} \leq \N{u}_{1} \leq 2C_{q}$ as in Lemma \ref{lem:l1_normalization}.
Setting $\Phi(x) := \Lambda(\Phi(x))$, the approximation error be decomposed into
\begin{align*}
\N{\Phi(x) - \eta(x)}_1 &= \N{\Lambda(\tilde \Phi(x)) -\frac{\tilde \Phi(x)}{\textN{\tilde \Phi(x)}_1}}_1 + \N{\frac{\tilde \Phi(x)}{\textN{\tilde \Phi(x)}_1}-\eta(x)}_1 \leq \frac{\varepsilon}{2} + \N{\frac{\tilde \Phi(x)}{\textN{\tilde \Phi(x)}_1}-\frac{\eta(x)}{\textN{\tilde \eta(x)}_1}}_1.
\end{align*}
For the second term, by twice applying triangle inequalities and reusing \eqref{eq:aux_projection_2}, we obtain
\begin{align*}
\N{\frac{\tilde \Phi(x)}{\textN{\tilde \Phi(x)}_1} -  \frac{\tilde \eta(x)}{\N{\tilde \eta(x)}_1}}_1
&\leq \frac{\textN{\tilde \Phi(x) - \tilde \eta(x)}_1}{\N{\tilde \eta(x)}_1} + \frac{\textSN{\textN{\tilde \Phi(x)}_1 - \textN{\tilde \eta(x)}_1}}{ \N{\tilde \eta(x)}_1} \leq C_{q}\frac{10\SN{\CZ}\tilde \varepsilon }{(\reach \delta)^2}.
\end{align*}
Combining both bounds yields and setting $\tilde \varepsilon =  (\reach\delta)^2 \varepsilon/(20C_{q}\SN{\CZ})$ yields the result.

\noindent
Lastly, we bound the complexity of $\Phi$. Following the rules of ReLU compositions and combinations
in Lemma \ref{lem:cocatanation} and \ref{lem:relu_linear_combination}, and using the cardinality bound $\SN{\CZ}\lesssim C_{\CM}\delta^{-d}$
as in Lemma \ref{lem:auxiliary_results_diff_geom}, we have
\begin{align}
\nonumber
L(\Phi) &= L(\Lambda) + L(\tilde \Phi_1) \lesssim C_{q}^4 \log^2(C_{q}/\varepsilon) + \log\left(\frac{DC_{q}\SN{\CZ}}{(\reach \delta)^2\varepsilon}\right) \leq
C_{q}^4 \log^2\left(\frac{C_{q}}{\varepsilon}\right) + \log\left(\frac{DC_{q}C_{\CM}}{\reach^2 \delta^{d+2}\varepsilon}\right)\\
\nonumber
W(\Phi) &\lesssim\SN{\CZ}W(\tilde \Phi_1) \lesssim DC_{\CM}\delta^{-d},\\
\label{eq:bounds_architecture_1}
P(\Phi) &\lesssim P(\Lambda) + \SN{\CZ}P(\tilde \Phi_1) \lesssim C_{q}^4 \SN{\CZ}\log^2\left(\frac{C_{q}}{\varepsilon}\right)
+ \SN{\CZ} D \log\left(\frac{DC_{q}\SN{\CZ}}{\reach^2\delta^2\varepsilon}\right)\\
\nonumber
&\lesssim C_{q}^4 C_{\CM}\delta^{-d}\log^2\left(\frac{C_{q}}{\varepsilon}\right) + D\delta^{-d}\log\left(\frac{DC_{q}C_{\CM}}{\reach^2\delta^{2+d}\varepsilon}\right),\\
\nonumber
\boundParams{\Phi} &= \boundParams{\tilde \Phi_1}\lesssim \reach^{-2} \vee \delta^{-2}.
\end{align}
\end{proof}

Finally we combine Lemma \ref{lem:existence_of_eta_network} with the $\alpha$-H\"older property of $g$ to conclude the proof.

\paragraph{Proof of Theorem \ref{thm:approximating_function_projection}}
Let $\CZ := \{z_1,\ldots,z_K\}$ be a maximal separated $\varepsilon$-net  of $\CM$ with
$K :=\SN{\CZ}\lesssim C_{\CM}\varepsilon^{-d}$ by Lemma \ref{lem:auxiliary_results_diff_geom} and let
$g(Z) = (g(z_1),\ldots,g(z_K)) \in \bbR^{K}$. By Lemma \ref{lem:existence_of_eta_network},
we can construct a network $\Theta : \bbR^D \rightarrow \bbR^K$, which approximates the partition of unity
function $\eta(x)$ in \eqref{eq:POU_guarantee} over $\CA \subseteq \CM(q)$ up to accuracy $\varepsilon^{\alpha}$.
To approximate the target $f$ we define the net
\begin{align}
\label{eq:network_construction_projection}
\Psi(x) := \sum_{i=1}^{K}\relu{g(z_i)\Theta_i(x)} - \relu{-g(z_i)\Theta_i(x)} = \langle g(Z), \Theta(x)\rangle.
\end{align}
Taking arbitrary $x \in \CA \subseteq \CM(q)$,
we can first use triangle and H\"older inequalities to get
\begin{align*}
\SN{f(x)-\Phi(x)} &= \SN{g(\pi_{\CM}(x))- \langle g(Z), \Theta(x)\rangle} \\
&\leq
\SN{g(\pi_{\CM}(x))-\langle g(Z), \eta(x)\rangle} + \SN{\langle g(Z), \left(\eta(x)-\Theta(x)\right)\rangle}\\
&\leq \SN{\langle g(\pi_{\CM}(x))\mathbbm{1}_K - g(Z), \eta(x)\rangle} +\N{g(Z)}_{\infty} \N{\eta(x) - \Theta(x)}_{1}\\
&\leq \SN{\langle g(\pi_{\CM}(x))\mathbbm{1}_K - g(Z), \eta(x)\rangle} + \varepsilon^{\alpha}.
\end{align*}
where $\mathbbm{1}_K = [1,\ldots,1]\in \bbR^{K}$ and where we used $\langle \mathbbm{1}_K, \eta(x)\rangle = 1$, $ \N{g(Z)}_{\infty} \leq 1$.
The first term can be bounded by H\"older's inequality and Proposition \ref{prop:construction_partition_of_unity} according to
\begin{align*}
\SN{\langle g(\pi_{\CM}(x))\mathbbm{1}_K- g(Z), \eta(x)\rangle} \leq \sup_{\substack{i \in [K]\\ \eta_i(x)\neq 0}}\SN{g(\pi_{\CM}(x)) - g(z_i)} \leq L\sup_{\substack{i \in [K]\\ \eta_i(x)\neq 0}}d_{\CM}^{\alpha}(\pi_{\CM}(x), z_i)
\lesssim L\left(\frac{1}{(1-q)^2}\varepsilon\right)^{\alpha},
\end{align*}
which shows the approximation error bound. To bound the complexity of $\Phi$ we note
that the network is a composition of $\Theta$ with a two-layer network that has first layer
weights $\pm g(Z) \in [-1,1]^K$ and second layer weights $\pm 1$. Therefore, the complexity of $\Phi$
is dominated by the complexity of $\Theta$ and can be read off from
Lemma \ref{lem:existence_of_eta_network}, respectively, from \eqref{eq:bounds_architecture_1} in the proof. \qed

\begin{proof}[Proof of Corollary \ref{cor:approximation_of_projection_operator}]
For each $k \in [D]$ we can approximate $x\mapsto  e_k^\top \pi_{\CM}(x)$ via
a ReLU net $\Phi_k$ using Theorem \ref{thm:approximating_function_projection} for $g(\cdot) = e_k^\top (\cdot)$. By stacking
these networks, we obtain the approximating network $\Phi(x) =(\Phi_1(x),\ldots,\Phi_D(x))$, which
achieves the asserted guarantee. Note that by construction \eqref{eq:network_construction_proof_sketch} we have $\Phi_k(x) = \sum_{i=1}^{K} (z_i)_k \Theta_i(x)$.
Since $\Theta_i$ is independent of $k$,
corresponding weights can be shared in constructing $\Phi$, so that the network architecture adheres
to the bounds in Theorem \ref{thm:approximating_function_projection}.
\end{proof}

\section{Proof of Theorem \ref{thm:approximating_function_distance}}
\label{subsec:proof_approximation_distance}
Theorem \ref{thm:approximating_function_distance} follows by the ReLU composition rule
after approximating $g$ and $x\mapsto \min_{z \in \CC}\normaldist{x}{z}^\expmet$. Let us separately prove the latter
result now and then give the proof of Theorem \ref{thm:approximating_function_distance}.

\begin{lemma}
\label{lem:approximating_distance_operator}
Let $\CC \subseteq [0,1]^D$ be nonempty and closed,  $\normaldistnoargs : [0,1]^D \times [0,1]^D \rightarrow [0,1]$ a metric
satisfying \eqref{eq:assumption_distance_relu_approximable} for some $\expmet \geq 1$, and assume there exists $\delta_0 > 0$ so that
$\CP(\delta, \CC, \normaldistnoargs ) \lesssim \delta^{-d}$ for all $\delta \in (0,\delta_0)$.
For any $\varepsilon \in (0,2\expmet\delta_0)$ there exists a ReLU network $\Phi$ with $L(\Phi) \lesssim d\log(\expmet\varepsilon^{-1})+ L_{\normaldistnoargs}(\varepsilon)$,
$W(\Phi) \lesssim (\expmet/\varepsilon)^{d} W_{\normaldistnoargs}(\varepsilon)$, $P(\Phi) \lesssim (\expmet/\varepsilon)^{d}\left(d\log(\expmet\varepsilon^{-1}) + P_{\normaldistnoargs}(\varepsilon)\right)$,
and $\boundParams{\Phi} \leq 1\vee \boundParamsnoargs_{\normaldistnoargs}(\varepsilon)$ satisfying
\begin{equation}
\label{eq:distance_function_approx}
\sup_{x \in [0,1]^D}\SN{\min_{z \in \CC}\normaldist{x}{z}^\expmet - \Phi(x)} \leq 2 \varepsilon.
\end{equation}
\end{lemma}
\begin{proof}
Let $\CZ \subset \CC$ be a maximal separated $\varepsilon/\expmet$-net of $\CC$, which has cardinality
bounded according to $\SN{\CZ}\lesssim (\expmet/\varepsilon)^{d}$ as soon as $\varepsilon < \expmet\delta_0$.
For each $z \in \CZ$, let $\Psi_{z_i,\varepsilon}$ be a ReLU network that approximates $\normaldist{x}{z_i}^\expmet$ up to accuracy $\varepsilon$
and let $\Gamma : \bbR^{\SN{\CZ}} \rightarrow \bbR$ be a network that realizes
$\Gamma(u) = \min_{i \in [\SN{\CZ}]}u_i$ (see Lemma \ref{lem:relu_calculus_minimum}).
We set $\Phi(x) = \Gamma(\Psi_{z_1,\varepsilon},\ldots,\Psi_{z_K,\varepsilon})$.
Using the triangle inequality, we decompose
\begin{align}
\nonumber
\SN{\min_{z \in \CC}\normaldist{x}{z}^\expmet- \Phi(x)} &\leq \SN{\min_{z \in \CZ}\normaldist{x}{z}^\expmet -\Phi(x)} + \SN{\min_{z \in \CZ}\normaldist{x}{z}^\expmet - \min_{z \in \CC}\normaldist{x}{z}^\expmet}\\
\label{eq:initial_decomposition}
&\leq \SN{\min_{z \in \CZ}\normaldist{x}{z}^\expmet-\min_{z \in \CZ}\Psi_{z,\varepsilon}(x)} + \SN{\min_{z \in \CZ}\normaldist{x}{z}^\expmet - \min_{z \in \CC}\normaldist{x}{z}^\expmet}
\end{align}
For the first term, we immediately have
\begin{align*}
\SN{\min_{z \in \CZ}\normaldist{x}{z}^\expmet-\min_{z \in \CZ}\Psi_{z,\varepsilon}(x)} \leq \max_{z \in \CZ}\SN{\normaldist{x}{z}^\expmet-\Psi_{z,\varepsilon}(x)}\leq \varepsilon.
\end{align*}
For the second term in \eqref{eq:initial_decomposition} we note that there exists $v(x) \in \CC$ satisfying
$\normaldist{x}{v(x)} = \min_{z \in \CC}\normaldist{x}{z}$ because $\CC$ is closed and nonempty ($v(x)$ does need to be unique). Then,
by $\SN{a^\expmet - b^\expmet}\leq \expmet(a \vee b)^{(\expmet-1)}\SN{a-b}$, $\normaldist{\cdot}{\cdot}\in [0,1]$,
and the inverse triangle inequality
we have
\begin{align*}
\SN{\min_{z \in \CZ}\normaldist{x}{z}^\expmet - \min_{z \in \CC}\normaldist{x}{z}^\expmet} &= \SN{\min_{z \in \CZ}\normaldist{x}{z}^\expmet - \normaldist{x}{v(x)}^\expmet} \leq \expmet \min_{z \in \CZ}\normaldist{z}{v(x)}\leq \varepsilon,
\end{align*}
with the last inequality following by the $\varepsilon/\expmet$ covering property of $\CZ$.
It remains to bound complexity
of the network in terms of $\varepsilon$ and $\CZ$. Using the rules
of compositions and linear combinations of networks in Lemma \ref{lem:cocatanation} and \ref{lem:relu_linear_combination}
we have
\begin{align*}
L(\Phi) &= L(\Gamma) + L((\Psi_{z_1,\varepsilon},\ldots\Psi_{z_K,\varepsilon})) \lesssim \log(\SN{\CZ}) + L(\Psi_{z_1,\varepsilon/2}) \lesssim d\log(\expmet\varepsilon^{-1})+ L_{\normaldistnoargs}(\varepsilon)\\
W(\Phi) &= \max\{W(\Gamma), W((\Psi_{z_1,\varepsilon},\ldots\Psi_{z_K,\varepsilon})), 2\SN{\CZ}\} \lesssim \SN{\CZ} W(\Psi_{z_1,\varepsilon}) \lesssim (\expmet/\varepsilon)^{d} W_{\normaldistnoargs}(\varepsilon),\\
P(\Phi) &\lesssim P(\Gamma) + P((\Psi_{z_1,\varepsilon},\ldots\Psi_{z_K,\varepsilon}))\lesssim P(\Gamma) + \SN{\CZ}P(\Psi_{z_1,\varepsilon}) \\
&\lesssim \SN{\CZ}\log(\SN{\CZ}) + \SN{\CZ}P(\Psi_{z_1,\varepsilon}) \lesssim  (\expmet/\varepsilon)^{d}\left(d\log(\expmet\varepsilon^{-1}) + P_{\normaldistnoargs}(\varepsilon)\right),\\
\boundParams{\Phi} &\leq \boundParams{\Gamma} \vee \boundParams{(\Psi_{z_1,\varepsilon},\ldots\Psi_{z_K,\varepsilon}))} \leq 1 \vee \boundParamsnoargs_{\normaldistnoargs}(\varepsilon).
\end{align*}
\end{proof}

\noindent
To prove Theorem \ref{thm:approximating_function_distance} we now combine Lemma \ref{lem:approximating_distance_operator} with Theorem \ref{thm:smooth_function_approx}
in Section \ref{sec:preliminaries}, which provides approximation bounds for
univariate $\alpha$-H\"older functions like $g_1,\ldots, g_M$.

\paragraph{Proof of Theorem \ref{thm:approximating_function_distance}}
Consider the case $M = 1$ first and let $g_1 = g$, $\CC_1 = \CC$.
Let $\Psi : \bbR^D \rightarrow \bbR$ be the ReLU net approximating $x\mapsto \min_{z \in \CC}\normaldist{x}{z}^\expmet$ up to accuracy $\varepsilon$ according to Lemma
\ref{lem:approximating_distance_operator}, with $\varepsilon < 2p\delta_0$,  and let $\Theta: \bbR\rightarrow \bbR $ be a ReLU net
that realizes $\Theta(t) = 1\wedge t= 1 -\relu{1-t}$. Furthermore, by Theorem \ref{thm:smooth_function_approx} there exists a ReLU network $\Omega$ that approximates $g$ to accuracy
$L\varepsilon^{\alpha}$ over $[0,1]$. We define the overall approximation by
$\Phi(x) := \Omega(\Theta(\Psi(x)))$ and compute
\begin{align}
\nonumber
&\SN{g\left(\min_{z \in \CC}\normaldist{x}{z}^\expmet\right)-\Phi(x)} = \SN{g\left(\min_{z \in \CC}\normaldist{x}{z}^\expmet\right)-\Omega(\Theta(\Psi(x)))}\\
\nonumber
&\quad\quad\leq \SN{g\left(\min_{z \in \CC}\normaldist{x}{z}^\expmet\right)-g(\Theta(\Psi(x)))} +  \SN{g(\Theta(\Psi(x))) - \Omega(\Theta(\Psi(x)))}\\
\label{eq:decomposition_distance_aux}
&\quad\quad\leq \SN{g\left(\min_{z \in \CC}\normaldist{x}{z}^\expmet\right)-g(\Theta(\Psi(x)))} +  L\varepsilon^{\alpha},
\end{align}
where we used $\Theta(\Psi(x)) \in [0,1]$ by construction and the approximation guarantees about $\Omega$ in the last step.
For the first term in \eqref{eq:decomposition_distance_aux}, we use the $\alpha$-H\"older property
of $g$ to get
\begin{align*}
\SN{g\left(\min_{z \in \CC}\normaldist{x}{z}^\expmet\right) -  g(\Theta(\Psi(x)))} &\leq L \SN{\min_{z \in \CC}\normaldist{x}{z}^\expmet - \Theta(\Psi(x))}^{\alpha} = L\SN{\min_{z \in \CC}\normaldist{x}{z}^\expmet - 1 \wedge \Psi(x)}^{\alpha} \\
&\leq L \SN{\min_{z \in \CC}\normaldist{x}{z}^\expmet - \Psi(x)}^{\alpha} \leq L(2\varepsilon)^{\alpha}\lesssim L\varepsilon^{\alpha},
\end{align*}
where the second to last inequality is an equality if $\Psi(x) < 1$, and follows from $\normaldist{x}{z} \leq 1$
if $\Psi(x) \geq 1$.
To bound the complexity of $\Phi$ we will use the rules of compositions according to Lemma \ref{lem:cocatanation}.
We have $\boundParams{\Phi} \leq \max\{\boundParams{\Omega}, \boundParams{\Theta},\boundParams{\Psi} \} \leq 1 \vee \boundParamsnoargs_{\normaldistnoargs}(\varepsilon)$ and
\begin{align*}
L(\Phi) &\leq L(\Omega) + L(\Theta) + L(\Psi) \lesssim \log(\varepsilon^{-1}) +  d\log(p\varepsilon^{-1})+ L_{\normaldistnoargs}(\varepsilon) \lesssim
d\log(p\varepsilon^{-1}) + L_{\normaldistnoargs}(\varepsilon),\\
W(\Phi) &\leq \max\{W(\Omega), W(\Theta), W(\Psi)\} \lesssim \varepsilon^{-1} + p^d \varepsilon^{-d} W_{\normaldistnoargs}(\varepsilon) \lesssim \expmet^d\varepsilon^{-(1 \vee d)}W_{\normaldistnoargs}(\varepsilon),\\
P(\Phi) &\lesssim P(\Omega) + P(\Theta) + P(\Psi) \lesssim \log(\varepsilon^{-1})\varepsilon^{-1} + p^d\varepsilon^{-d}\left(d\log(p\varepsilon^{-1}) +
P_{\normaldistnoargs}(\varepsilon)\right)\\
& \lesssim p^dd\log(p\varepsilon^{-1})\varepsilon^{-(1\vee d)} + p^d\varepsilon^{- d}P_{\normaldistnoargs}(\varepsilon).
\end{align*}
For the case $M > 1$ we construct networks $\Phi_{\ell}$ approximating $g_{\ell}(\normaldist{x}{\CC_{\ell}})$
to accuracy $L \varepsilon^{\alpha}$ each, and then use $x \mapsto \sum_{i=1}^{M}\Phi_{\ell}(x)$, which can be realized
by a ReLU net according to Lemma \ref{lem:relu_linear_combination}. The error follows from the triangle inequality and
the dimensions can be deduced from Lemma \ref{lem:relu_linear_combination}.
\qed

\section{Conclusion and future directions}
\label{sec:conclusions}
In this work we study the uniform approximation of certain compositional functions by deep ReLU networks.
The considered function classes are motivated by practical examples and generalize some frequently studied function classes,
including functions defined on low-dimensional domains.
We have proven uniform approximation guarantees with moderately deep networks,
a near-optimal dependency on the number of nonzero network parameters, and optimal
dependency on the number of required function queries. Our results suggest that local invariances
encoded in the mapping $x \mapsto f(x)$ drive the approximation complexity rather than
the complexity of the domain of the target.

We plan to extend our guarantees to projection-based functions $f(x) = g(\tilde \pi_{\CM}(x))$
using projections $\tilde \pi_{\CM}(x) = \argmin_{x \in \CM}d(x,z)$ based on other metrics $d$ and
less regular sets $\CM$.
This allows for considering  more general nonlinear reduction maps $\phi$
and thus further enhances our knowledge about the adaptivity of deep networks. Furthermore
we plan to study the influence of the domain of the target (or more practically a given data set)
 on the training process of deep networks. While approximability
is not crucially dependent on the data domain according to our results,
training deep networks via backpropagation may still be affected by the domain of the data.

\section{Appendix}
\label{sec:appendix}
\subsection{Proof of Lemma \ref{lem:combined_unique_projection_lipschitz_property}}
\label{subsec:proofs_sec_2}
For the proof we recall simplified version of \cite[Theorem 4.8]{federer1959curvature} tailored to manifolds.
\begin{theorem}[{\cite[6) and 7) in Theorem 4.8]{federer1959curvature}}]
\label{thm:federer_reference}
Let $\CM \subset \bbR^D$ be a compact submanifold of $\bbR^D$.

\noindent
1) Let $v \in \CM$ and $x \in \bbR^D$ so that $\sup_{t\geq 0}\{\pi_{\CM}(v+tx) = v\} \in (0,\infty)$,
then $v + rx \not \in \textrm{Int}(\medial^C)$.

\noindent
2) Let $x \in \medial^C$ and $\lreach{\pi_{\CM}(x)} > 0$. Then for any $z \in \CM$ we have
\begin{align*}
\langle x-\pi_{\CM}(x), \pi_{\CM}(x) - z\rangle \geq - \frac{\N{\pi_{\CM}(x) - z}_2^2\N{x-\pi_{\CM}(x)}}{2\lreach{\pi_{\CM}(x)}}.
\end{align*}
\end{theorem}

\begin{proof}[Proof of Lemma \ref{lem:combined_unique_projection_lipschitz_property}]
\textbf{Part 1:} We first note that $\dist{x}{\CM} \leq \N{x - v}_2 < q \lreach{v}\leq \lreach{v}$, which implies
$x \not\in \medial$, and thus there exists a unique projection $\pi_{\CM}(x)$ according to the construction of $\medial$.
To show $\pi_{\CM}(x) = v$, we consider
a proof by contradiction. Assume $\pi_{\CM}(x) \neq v$ and denote
$$
l := \sup_{t \geq 0}\left\{\pi_{\CM}\left(v+t\frac{u}{\N{u}_2}\right) = v\right\}.
$$
We have $l > 0$, since $u \perp \Im(A(v))$ and $\reach > 0$ (see for instance \cite[Section 4]{niyogi2008finding}),
and $l < q\lreach{v}$, since $\pi_{\CM}(x) \neq v$. By part 1) in Theorem \ref{thm:federer_reference}
we get $w := v + l\frac{u}{\N{u}_2} \not\in \textrm{Int}(\medial^C)$. Therefore, for any $\varepsilon > 0$
there exists, with $B_{\varepsilon}(w)$ being a Euclidean ball of radius $\varepsilon$ around $w$,
$$
y \in B_{\varepsilon}(w) \cap \left(\textrm{Int}(\medial^C)\right)^C =  B_{\varepsilon}(w) \cap \textrm{cl}(\medial).
$$
Using the existence of such a $y$ for every $\varepsilon > 0$, we get
\begin{align*}
\lreach{v}  \leq \N{v-y}_2 &\leq  \N{v-w}_2 + \N{w - y}_2 \leq \N{v-x}_2 + \N{w - y}_2 < q\lreach{v} + \varepsilon.
\end{align*}
Letting $\varepsilon \rightarrow 0$ and recalling $q<1$, this is a false statement.

\noindent
\textbf{Part 2:}
Using part 2) of Theorem \ref{thm:federer_reference}, we have for any $x \in \CM(q)$ and $v \in \CM$
\begin{align}
\label{eq:auxilairy_result_reach_lipschitz}
\langle x-\pi_{\CM}(x), \pi_{\CM}(x) - v\rangle \geq - \frac{\N{\pi_{\CM}(x) - v}_2^2\N{x-\pi_{\CM}(x)}}{2\lreach{\pi_{\CM}(x)}}.
\end{align}
Taking arbitrary $x,x' \in \CM(q)$ we obtain by the Cauchy-Schwartz inequality and \eqref{eq:auxilairy_result_reach_lipschitz}
\begin{align*}
&\N{x-x'}_2\N{\pi_{\CM}(x)-\pi_{\CM}(x')}_2 \geq \langle x-x', \pi_{\CM}(x)-\pi_{\CM}(x')\rangle\\
&\qquad=\langle x-\pi_{\CM}(x) + \pi_{\CM}(x)-\pi_{\CM}(x')+\pi_{\CM}(x')-x', \pi_{\CM}(x)-\pi_{\CM}(x')\rangle\\
&\qquad\geq \N{\pi_{\CM}(x)-\pi_{\CM}(x')}_2^2\left(1-\frac{1}{2}\frac{\N{x-\pi_{\CM}(x)}_2}{\lreach{\pi_{\CM}(x)}}-\frac{1}{2}\frac{\N{x'-\pi_{\CM}(x')}_2}{\lreach{\pi_{\CM}(x')}}\right) \\
&\qquad= \N{\pi_{\CM}(x)-\pi_{\CM}(x')}_2^2(1-q),
\end{align*}
where we used $x,x' \in \CM(q)$ in the last inequality.
\end{proof}

\subsection{Additional result from ReLU calculus}
\label{subsec:relu_calculus}
In this section we prove the approximation guarantees listed in Table \ref{tab:relu_calculus}.

\begin{lemma}[$p$-th power of $L_p$-norm]
\label{lem:relu_lp_norm}
Let $p \in \bbN$, $\varepsilon, R > 0$. There exists a ReLU network $\Phi$ with
$L(\Phi) \lesssim p^2\log(\lceil R\rceil D/\varepsilon)$, $W(\Phi) \leq 9D$, $P(\Phi) \lesssim D p^2\log(\lceil R\rceil D/\varepsilon)$
and $\boundParams{\Phi}\leq 1$ such that
\begin{align*}
\sup_{x \in [-R,R]^D} \SN{\N{x}_p^p - \Phi(x)} \leq \varepsilon.
\end{align*}
Furthermore, $\N{x}_1$ can be realized exactly with $L(\Phi) = 2$, $W(\Phi)=2D$, $P(\Phi) = 4D$ and $\boundParams{\Phi} = 1$.
\end{lemma}
\begin{proof}
Following part 3) in Lemma \ref{lem:grohs_summary}, there exists a ReLU network $\Gamma$ that approximates
$t \mapsto t^p$ to accuracy $\varepsilon D^{-1}> 0$ on $[-R,R]$.
Set $\Phi(x) := \sum_{i=1}^{D} \Gamma(x_i)$. For arbitrary $x \in [-R,R]^D$ we have
\begin{align*}
\SN{\N{x}_p^p - \Phi(x)} \leq \sum_{i=1}^{D}\SN{x_i^p - \Theta(x_i)} \leq \varepsilon.
\end{align*}
The complexity of $\Phi$ is bounded according to the rules in Lemma \ref{lem:cocatanation}, \ref{lem:relu_linear_combination}
and the bounds in Lemma \ref{lem:grohs_summary} for the network $\Gamma$. We obtain
$\boundParams{\Phi} \leq \max\{1,\boundParams{\Gamma}\} = 1$, $W(\Phi) = D (2\vee W(\Gamma)) \leq 9D$, and
\begin{align*}
L(\Phi) &= L(\Gamma) \lesssim p (\log(D/\varepsilon) + p\log(\lceil R \rceil)) \lesssim p^2\log(\lceil R\rceil D/\varepsilon) \\
P(\Phi) &= D (P(\Gamma) + W(\Gamma) + 1) \lesssim DL(\gamma) \leq D p^2\log(\lceil R\rceil D/\varepsilon).
\end{align*}
For $p = 1$ we notice $\N{x}_1 = \sum_{i=1}^{D}\SN{x_i} = \sum_{i=1}^{D}\relu{x_i} - \relu{-x_i}$,
which defines a shallow network with width $2D$ and $4D$ nonzero parameters.
\end{proof}

\begin{lemma}[Multiplication]
\label{lem:multiplication_network}
Let $\varepsilon \in (0,\frac{1}{2})$ and $a > 0$. There exists a ReLU network $\Phi : \bbR^D \times \bbR \rightarrow \bbR^D$
with $L(\Phi)\lesssim \log(a^2 \varepsilon^{-1})$, $W(\Phi)\leq 5 D$, $P(\Phi) \lesssim D\log(a^2 \varepsilon^{-1})$
and $\boundParams{\Phi}\leq 1$ with
\begin{align*}
\sup_{\N{x}_{\infty} \leq a,\ \SN{y}\leq a}\N{\Phi(x,y) - xy}_{\infty}\leq \varepsilon.
\end{align*}
\end{lemma}
\begin{proof}
By part b) of Lemma \ref{lem:grohs_summary} there exists a ReLU net $\Psi : \bbR^2\rightarrow \bbR$
approximating $xy$ up to accuracy $\varepsilon$ on $[-a,a]^2$. We set
$\Phi(x,y) = (\Psi(x_1,y),\ldots, \Psi(x_D,y))$, which can be realized by a ReLU net (Lemma \ref{lem:relu_linear_combination}).
Furthermore, using dimension bounds in Lemma \ref{lem:grohs_summary},
we get $L(\Phi) = L(\Psi) \lesssim \log(a^2 \varepsilon^{-1})$, $W(\Phi) \leq DW(\Psi) \leq 5D$,
$P(\Phi) = D\left(P(\Psi) + W(\Psi) + 1\right) \lesssim D\log(a^2 \varepsilon^{-1})$ and $\boundParams{\Phi} \leq 1 \vee \boundParams{\Psi} \leq 1$.
\end{proof}
\begin{lemma}[Division]
\label{lem:division}
Let $\varepsilon \in (0,1)$ and $a \in \bbR_{\geq 1}$. There exists a network $\Phi : \bbR \rightarrow \bbR$
with $L(\Phi)\lesssim a^4\log^2(a/\varepsilon)$, $W(\Phi) \leq 9$,
$P(\Phi) \lesssim a^4\log^2(a/\varepsilon)$ and $\boundParams{\Phi} \leq 1$,
so that
\begin{align*}
\sup_{t \in \left[\frac{1}{a},a\right]}\SN{\Phi(t) - \frac{1}{t}} \leq \varepsilon .
\end{align*}
\end{lemma}
\begin{proof}
We follow the proof strategy of \cite[Lemma 3.6]{telgarsky2017neural} but combine it with part c) of Lemma \ref{lem:grohs_summary}.
Set $c = \frac{1}{a}$ and $r = \lceil a^2\ln(\frac{2a}{\varepsilon })\rceil$. First, we notice $t^{-1} = c\sum_{i=1}^{\infty}(1-ct)^i$
so cutting the series at $i=r$ results in the approximation error
\begin{align*}
\SN{\frac{1}{t} - c \sum_{i=1}^{r}(1-ct)^i} = \SN{c\sum_{i=r+1}^{\infty}(1-ct)^i} \leq \frac{\varepsilon}{2}.
\end{align*}
Now let $p(t) = c\sum_{i=1}^{r}z^i$ so that $p(1-ct) = c\sum_{i=1}^{r}(1-ct)^i$ and notice that $0 \leq 1-ct \leq 1$ since $t \in [a^{-1},a]$ and
$c = a^{-1}$. Using part c) of Lemma \ref{lem:grohs_summary}, we can approximate
$p$ over $[0,1]$ to accuracy $\frac{\varepsilon}{2}$ with a network $\Psi$ adhering to the dimension bounds
$L(\Psi) \lesssim r\log(1/\varepsilon) + r^2 + r\log(\lceil c\rceil)$, $W(\Psi)\leq 9$, $P(\Psi) \lesssim L(\Psi)$, and $\boundParams{\Psi} \leq 1$.
Therefore, we get for any $t \in [a^{-1},a]$
\begin{align*}
\SN{\frac{1}{t}-\Psi(1-ct)} \leq \SN{\frac{1}{t}- p(1-ct)} + \SN{p(1-ct) - \Psi(1-ct)}\leq \frac{\varepsilon}{2} + \frac{\varepsilon}{2}\leq\varepsilon.
\end{align*}
We can simplify the bounds on $L(\Psi)$ and thus $P(\Psi)$ by recognizing that
$r^2 \asymp a^4\log^2(a/\varepsilon)$ dominates the terms $r\log(1/\varepsilon)$
and $r\log(\lceil c\rceil)$.
\end{proof}

\begin{lemma}[$L_1$-normalization]
\label{lem:l1_normalization}
Let $a\geq 1$, $\varepsilon \in (0,\frac{1}{2})$. There exists a ReLU network $\Phi : \bbR^D\rightarrow \bbR^D$
with $L(\Phi)\lesssim a^4\log^2\left(\frac{a}{\varepsilon}\right)$, $W(\Phi) \lesssim D$, $P(\Phi)\lesssim a^4 D \log^2\left(\frac{a}{\varepsilon}\right)$,
and $\boundParams{\Phi}\leq 1$ such that
\begin{align*}
\sup_{\frac{1}{a}\leq \N{x}_1 \leq a}\N{\Phi(x) - \frac{x}{\N{x}_1}}_{\infty} \leq \varepsilon.
\end{align*}
\end{lemma}
\begin{proof}
We combine four networks: a network realizing the identity,
a network realizing the $1$-norm, a network realizing approximate division based on Lemma \ref{lem:division},
and  a network realizing approximate multiplication based on Lemma \ref{lem:multiplication_network}.
The identity map $\Id_D : \bbR^D \rightarrow \bbR^D$ can be realized by a two-layer net $\Psi(x) = \relu{x}-\relu{-x}$
and $x \mapsto \N{x}_1$ can be realize by a two-layer ReLU net $\Theta(x) = \sum_{i=1}^{D}\relu{x_i}+\relu{-x_i}$.
Furthermore, let $\Gamma$ denote a ReLU net aproximating univariate division on $[a^{-1},a]$ up to accuracy
$\frac{\varepsilon}{2a}$, whose existence has been shown in Lemma \ref{lem:division}, and let
$\Omega$ denote a ReLU net approximating $(x,y) \mapsto yx$ on $[-2a,2a]^{D+1}$ to accuracy $\frac{\varepsilon}{2}$. Then we set
$\Phi(x) = \Omega(\Psi(x), \Gamma(\Theta(x)))$, which satisfies
\begin{align*}
\N{\Phi(x) - \frac{x}{\N{x}_1}}_{\infty} &\leq \N{\Omega(x, \Gamma(\N{x}_1)) - x \Gamma(\N{x}_1)}_{\infty} + \N{x \Gamma(\N{x}_1) - \frac{x}{\N{x}_1}}_{\infty}\\
&\leq \frac{\varepsilon}{2} + \N{x}_{\infty}\frac{\varepsilon}{2a} \leq \frac{\varepsilon}{2} + \N{x}_{1}\frac{\varepsilon}{2a} \leq \varepsilon,
\end{align*}
where we used $\N{x}_1 \leq a$ in the last inequality. To compute the dimensions of $\Phi$,
first note that the composition rules in Lemma \ref{lem:cocatanation} imply $\boundParams{\Gamma\circ \Theta} = \boundParams{\Gamma} \vee \boundParams{\Theta} \leq 8 \vee a^{-1}$
and
\begin{align*}
L(\Gamma\circ \Theta) &= L(\Theta) + L(\Gamma) \lesssim 2 + a^2\log^2\left(\frac{a}{\varepsilon}\right) \lesssim a^2\log^2\left(\frac{a}{\varepsilon}\right),\\
W(\Gamma\circ \Theta) &= \max\{W(\Theta), W(\Gamma), 2\} \leq 2D\vee 16,\\
P(\Gamma\circ \Theta) &= 2P(\Gamma) + 2P(\Theta) \lesssim a^2\log^2\left(\frac{a}{\varepsilon}\right) + D.
\end{align*}
Then, using linear combination and concanation rules of ReLU nets in Lemma  \ref{lem:cocatanation}, \ref{lem:relu_linear_combination} we obtain
\begin{align*}
L(\Phi) &= L(\Omega) + L((\Psi(x), \Gamma\circ \Theta)) \lesssim \log\left(a^2/\varepsilon\right) + 2 \vee a^4\log^2\left(a/\varepsilon\right) \lesssim a^4\log^2\left(a/\varepsilon\right),\\
W(\Phi) &= W(\Omega) \vee W((\Psi(x), \Gamma\circ \Theta)) \leq 5 D \vee (4 + W(\Psi) + W(\Gamma\circ \Theta))  \lesssim D,\\
P(\Phi) &\lesssim P(\Omega) + P((\Psi(x), \Gamma\circ \Theta)) \lesssim P(\Omega) + P(\Psi) + P(\Gamma\circ \Theta) + L(\Gamma\circ \Theta) + W(\Psi) + W(\Gamma\circ \Theta)\\
&\lesssim D\log(a^2\varepsilon^{-1}) + D + a^4\log^2\left(\frac{a}{\varepsilon}\right) \lesssim a^4 D \log^2 \left(\frac{a}{\varepsilon}\right),\\
\boundParams{\Phi} &= \boundParams{\Omega} \vee \boundParams{(\Psi, \Gamma\circ \Theta)} \leq \max\{1, \boundParams{\Psi}, \boundParams{\Gamma}, \boundParams{\Theta}\}\leq 1.
\end{align*}
\end{proof}

\begin{lemma}
\label{lem:relu_calculus_minimum}
Let $K\geq 2$. There exists a ReLU network $\Phi_K : \bbR^K \rightarrow \bbR$ with $L(\Phi_K) \leq 2\lceil\log_2(K)\rceil$, $W(\Phi_K)\leq 3\lceil K/2\rceil$,
$P(\Phi_K) \leq 11 K\lceil\log_2(K)\rceil$ and $\boundParams{\Phi_K} \leq 1$ such that $\Phi_K(x) = \min_{i \in [K]} x_i$.
\end{lemma}

\begin{proof}
Without loss of generality we assume $K$ is even as we can otherwise just replace $x$ by repeating
one of its arguments without changing the bounds on the dimension of the network.
We proof the statement by induction. For $K = 2$ define a network
$$\Phi_2(x) = \relu{x_1} - \relu{-x_1} - \relu{x_1 - x_2} = x_1 - \relu{x_1 - x_2} = x_1 \wedge x_2.$$ Clearly,
$L(\Phi_2) = 2$, $W(\Phi_2) = 3$, $P(\Phi_2) = 7$, and $\boundParams{\Phi_2} = 1$, which proves
the induction start. For the induction step $(K-1)\rightarrow K$ we assume the statement holds up to $K-1$ and we set
$\Phi_K = \Phi_{\frac{K}{2}}(\Phi_2(x_1,x_2),\ldots,\Phi_{2}(x_{K-1},x_{K}))$, which
realizes $\min_{x \in [K]}x_i$. To compute the network complexity we use composition and
parallelization rules from Lemma \ref{lem:cocatanation}, \ref{lem:relu_linear_combination}. This gives
$\boundParams{\Phi_K} \leq 1$ and
\begin{align*}
L(\Phi_K) &= L\left(\Phi_{\frac{K}{2}}\right) + L(\Phi_2) = 2\left\lceil\log_2\left(\frac{K}{2}\right)\right\rceil + 2
=2\left\lceil\log_2\left(K\right) - 1\right\rceil + 2 =2\left\lceil\log_2\left(K\right)\right\rceil,\\
W(\Phi_K) &= \max\left\{W(\Phi_{\frac{K}{2}}, W(\Phi_2,\ldots,\Phi_{2}), K\right\}\leq \frac{K}{2}W(\Phi_{2})\leq 3\frac{K}{2},\\
P(\Phi_K) &= 2P\left(\Phi_\frac{K}{2}\right) + 2 P(\Phi_2,\ldots,\Phi_{2}) \leq 11K\left\lceil\log_2\left(\frac{K}{2}\right)\right\rceil + K(P(\Phi_2) + W(\Phi_2) + 1)\\
&\leq 11K\left\lceil \log_2\left(K\right)\right\rceil - 11K + 11K.
\end{align*}
\end{proof}

\begin{remark}
\label{rem:L_infty_appprox}
The $L_{\infty}$-norm can be realized by a ReLU net due to
Lemma \ref{lem:relu_calculus_minimum} and the identity
\begin{align*}
\N{x}_{\infty} = \max_{i \in [D]}\SN{x_i} = \max_{i \in [D]}\relu{x_i} + \relu{-x_i} = -\min_{i \in [D]}-(\relu{x_i} + \relu{-x_i}),
\end{align*}
\end{remark}

\section*{Funding}
\noindent
AC is supported by NSF DMS grants 1819222 and 2012266, and by Russell Sage Foundation Grant 2196.


\end{document}